\documentclass[11pt,letterpaper]{article}
\usepackage{xcolor}
\usepackage{epsfig}
\usepackage[margin=1in]{geometry}
\usepackage{amssymb, amsmath, amsthm, graphicx,subcaption, mathtools,tikz-cd}

\usepackage[T1]{fontenc}
\usepackage[section]{algorithm}
\usepackage{enumitem}
\usepackage{nicematrix, tikz}
\usepackage[dvipsnames]{xcolor}
\usepackage{hyperref}
\usepackage{algpseudocode}

\hypersetup{
  colorlinks = true,
  linkcolor  = RoyalBlue,
  citecolor  = ForestGreen,
  urlcolor   = Magenta
}

\makeatletter
\newcommand{\mytag}[2]{
  \text{#1}
  \@bsphack
  \begingroup
    \@onelevel@sanitize\@currentlabelname
    \edef\@currentlabelname{
      \expandafter\strip@period\@currentlabelname\relax.\relax\@@@
    }
    \protected@write\@auxout{}{
      \string\newlabel{#2}{
        {#1}
        {\thepage}
        {\@currentlabelname}
        {\@currentHref}{}
      }
    }
  \endgroup
  \@esphack
}
\makeatother

\DeclareMathOperator*{\argmax}{arg\,max}

\newcommand{\prob}{\mathbb{P}}
\newcommand{\E}{\mathbb{E}}

\newcommand{\R}{\mathbb{R}}

\newcommand{\N}{\mathbb{N}}

\newcommand{\I}{\mathbb{I}}
\newcommand{\horizon}{\mathfrak{T}}

\newcommand{\F}{\mathcal{F}}
\newcommand{\NN}{\mathcal{N}}
\newcommand{\LL}{\mathcal{L}}

\newcommand{\TT}{\mathcal{T}}
\newcommand{\UU}{\mathcal{U}}
\newcommand{\regret}{\mathcal{R}}
\newcommand{\EE}{\mathcal{E}}
\newcommand{\GG}{\mathcal{G}}
\newcommand{\muhat}{\widehat{\mu}}
\newcommand{\sample}{\theta}
\newcommand{\Ncal}{\mathcal{N}}
\newcommand{\forevery}{\; \text{for all} \;\; }
\newcommand{\nstar}{n^\star}
\newcommand{\inprob}{\overset{\prob}{\longrightarrow}}
\newcommand{\indist}{\overset{\mathcal{D}}{\longrightarrow}}
\newcommand{\sigmahat}{\widehat{\sigma}}

\newcommand{\mustar}{\mu^\star}
\newcommand{\Exs}{\mathbb{E}}
\newcommand{\X}{X}
\newcommand{\algo}{\ensuremath{\mathcal{A}}}
\newcommand{\CI}{\text{CI}}
\newcommand{\gap}{\Delta}

\theoremstyle{plain}
\newtheorem{theorem}{Theorem}

\newtheorem{lemma}{Lemma}
\newtheorem{definition}{Definition}
\newtheorem{proposition}{Proposition}

\colorlet{shadecolor}{gray!10}
\colorlet{lightyellow}{yellow!10}

\title{\LARGE \textbf{Stable Thompson Sampling:\\
Valid Inference via Variance Inflation}}

\author{
\large Budhaditya Halder$^{\dagger}$ \quad 
Shubhayan Pan$^{\ddagger}$ \quad 
Koulik Khamaru$^{\dagger}$ \\
\vspace{0.2cm}
\normalsize $^{\dagger}$Department of Statistics, Rutgers University \\
\normalsize $^{\ddagger}$Indian Statistical Institute, Kolkata
}
\date{}
\begin{document}

\maketitle

\begin{abstract}%
  We consider the problem of statistical inference when the data is collected via a Thompson Sampling-type algorithm. While Thompson Sampling (TS) is known to be both asymptotically optimal and empirically effective, its adaptive sampling scheme poses challenges for constructing confidence intervals for model parameters. We propose and analyze a variant of TS, called Stable Thompson Sampling, in which the posterior variance is inflated by a logarithmic factor. We show that this modification leads to asymptotically normal estimates of the arm means, despite the non-i.i.d. nature of the data. Importantly, this statistical benefit comes at a modest cost: the variance inflation increases regret by only a logarithmic factor compared to standard TS. Our results reveal a principled trade-off: by paying a small price in regret, one can enable valid statistical inference for adaptive decision-making algorithms.%
\end{abstract}

\vspace{-10pt}
\section{Introduction}
\label{sec:intro}
Multi-armed bandit algorithms form a foundational framework for sequential decision-making under uncertainty, wherein a learner repeatedly selects among competing actions in order to maximize cumulative reward. A key statistical challenge in this setting is the inherent exploration–exploitation trade-off: how to balance the need to acquire information about uncertain actions (exploration) with the goal of leveraging current knowledge to accrue reward (exploitation). Owing to their conceptual simplicity and wide applicability, bandit methods have been deployed across a range of domains, including personalized recommendation~\cite{zhou2017large,bouneffouf2012contextual,bouneffouf2013contextual}, online advertising~\cite{wen2017online,vaswani2017model}, adaptive clinical trials~\cite{bastani2020online}, and financial portfolio optimization~\cite{shen2015portfolio,huo2017risk}, to name a few. For a comprehensive overview of applications of bandit algorithms, see the survey by \cite{bouneffouf2020survey}.

A variety of algorithmic strategies have been developed to address the exploration–exploitation dilemma. Among the most prominent are the Upper Confidence Bound (UCB) family~\cite{lairobbins,lai1987adaptive,auer2002finite}, which constructs high-probability upper confidence bounds on the mean rewards of each arm and selects the arm with the highest such bound; $\varepsilon$-- greedy approaches~\cite{sutton1998reinforcement}, which inject random exploration with fixed or decaying probability; and exponential–weighted algorithms~\cite{auer2002nonstochastic}, which assign probabilities to actions via multiplicative updates, among other commonly used strategies. Within this landscape, Thompson Sampling (TS), originally proposed by \cite{thompson1933likelihood}, occupies a distinct position due to its Bayesian formulation. At each decision point, TS samples a mean-reward parameter for each arm from its posterior distribution of arm means given the observed data, and selects the arm with the highest sampled mean parameter. In practice, Thompson Sampling has gained widespread adoption not only due to its Bayesian interpretability but also because of its algorithmic simplicity and favorable empirical performance~\cite{chapelle2011empirical}.

While bandit algorithms are traditionally motivated by the goal of reward maximization, many contemporary applications demand more than just optimal decision-making—they require valid statistical inference~\cite{berry2012adaptive,trella2024oralytics,trella2022designing,nahum2024optimizing,chow2005statistical}. In such settings, the adaptive nature of data collection presents a fundamental challenge, as the resulting data sequence no longer satisfies the standard i.i.d. assumptions that underpin much of classical statistical theory. In particular, in the multi-armed bandit problem, sample means—computed from data collected by adaptive algorithms—may not converge to an normal distribution asymptotically~\cite{zhang2020inference}. This lack of normality complicates the use of classical tools such as confidence intervals and hypothesis tests.

In this paper, we study how to construct confidence intervals for model parameters in multi-armed bandit problems where actions are selected using the Thompson Sampling algorithm.
\vspace{-.1cm}
\subsection{Related work}
The difficulties of conducting valid inference under adaptive data collection are not new; analogous issues have been documented in time series and econometrics, where autoregressive dependencies can invalidate classical asymptotic approximations~\cite{dickey1979distribution,white1958limiting,white1959limiting,lai1982least}. In the context of bandits and sequential experimentation, these challenges manifest in subtle ways. Standard estimators such as sample means may fail to converge in distribution to their classical Gaussian limits when data is collected via adaptive algorithms~\cite{zhang2020inference,deshpande2018accurate,deshpande2023online,khamaru2021near,lin2023semi,lin2024statistical,ying2024adaptive}. This phenomenon underscores the need for new inferential frameworks that explicitly account for adaptive sampling.

Two complementary lines of work have emerged. The first is \textbf{non-asymptotic} in nature, leveraging concentration inequalities for self-normalized Martingales~\cite{abbasi2011improved, howard2020timeuniformchernoffboundsnonnegative, shin2019bias,waudby2023anytime}. Building on the foundational analyses of \cite{de2004self,pena2008self} these methods yield confidence intervals that hold uniformly over time, albeit often at the expense of conservativeness. The second line is \textbf{asymptotic}, aiming to recover classical inference tools by exploiting the Martingale structure inherent in adaptively collected data. This includes the use of Martingale central limit theorems~\cite{hall2014martingale} and debiasing techniques~\cite{zhang2014confidence,zhang2020inference,hadad2021confidence,bibaut2021post,zhan2021off,syrgkanis2023post}, which produce tighter confidence intervals in large-sample regimes. The line of work that is closest to our work is~\cite{kalvit2021closer,fan2022typical,khamaru2024inference,han2024ucb}. Most relevant to our work are recent developments in~\cite{kalvit2021closer,khamaru2024inference,han2024ucb}, which show that for Upper Confidence Bound (UCB) algorithms, certain stability properties—grounded in a deterministic characterization of the number of arm pulls—can restore asymptotic normality of the sample means and enable valid classical inference under adaptive sampling. In this work, we extend this perspective to a modified variant of Thompson Sampling, showing that a similar deterministic analysis yields rigorous asymptotic inferential guarantees—albeit via proof techniques that differ substantially from those used in the UCB setting.
\vspace{-.4cm}
\subsection{Contributions}
Our goal is to enable valid statistical inference when data are collected using Thompson Sampling. We begin by analyzing the behavior of the classical algorithm in a two-armed Gaussian bandit problem. As illustrated in Figure~\ref{fig:ts-sts-comparison}, the sample means under Thompson Sampling are not asymptotically normal, and the resulting confidence intervals systematically under-cover the true means.
\begin{figure}[htbp]
    \centering
    \hfill
    \begin{minipage}[t]{0.295\textwidth}
        \centering
        \includegraphics[trim=0 30 0 0, clip,width=\linewidth]{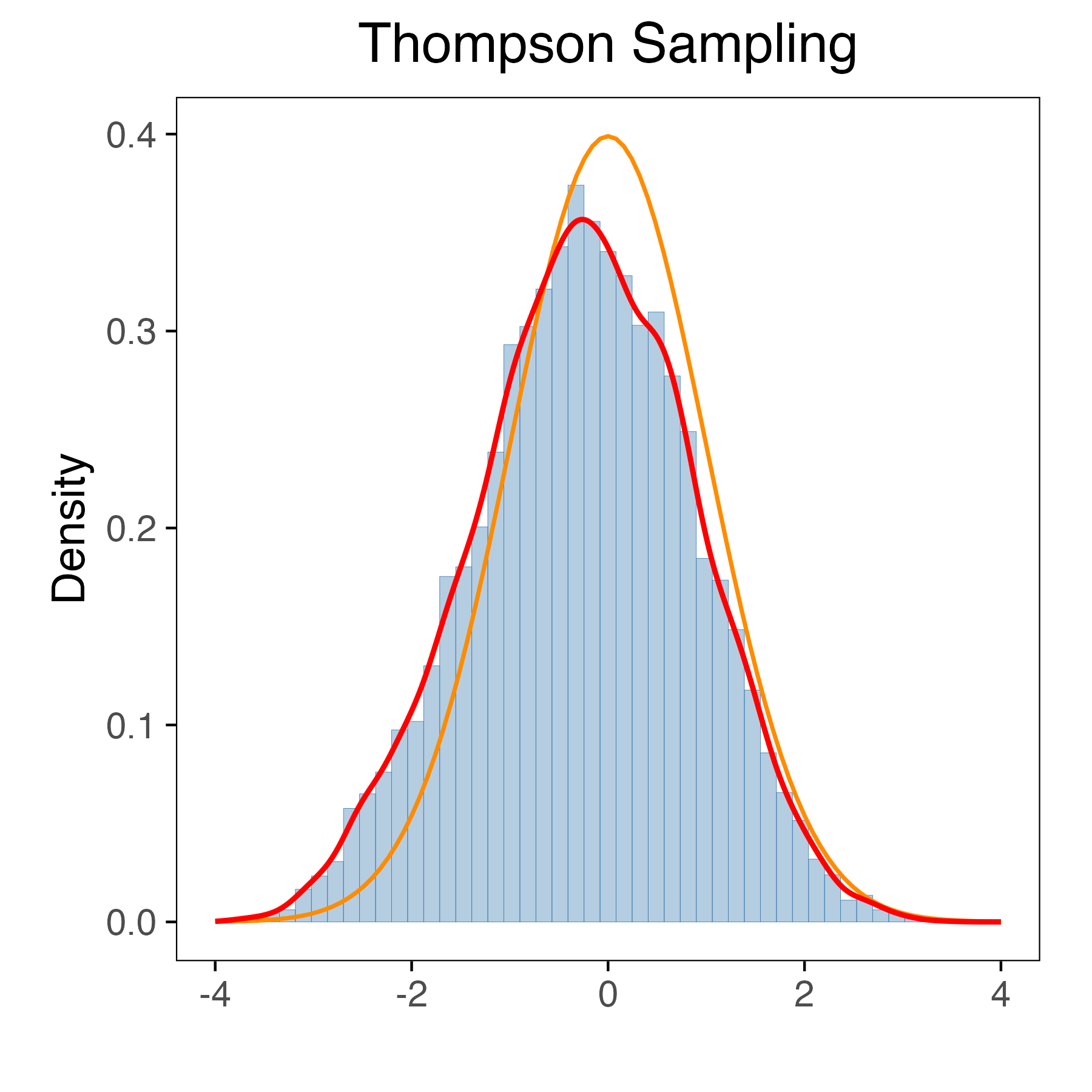}
        \caption*{\qquad Standardized error}
    \end{minipage}
    \hfill
    \begin{minipage}[t]{0.295\textwidth}
        \centering
        \includegraphics[trim=0 30 0 0, clip,width=\linewidth]{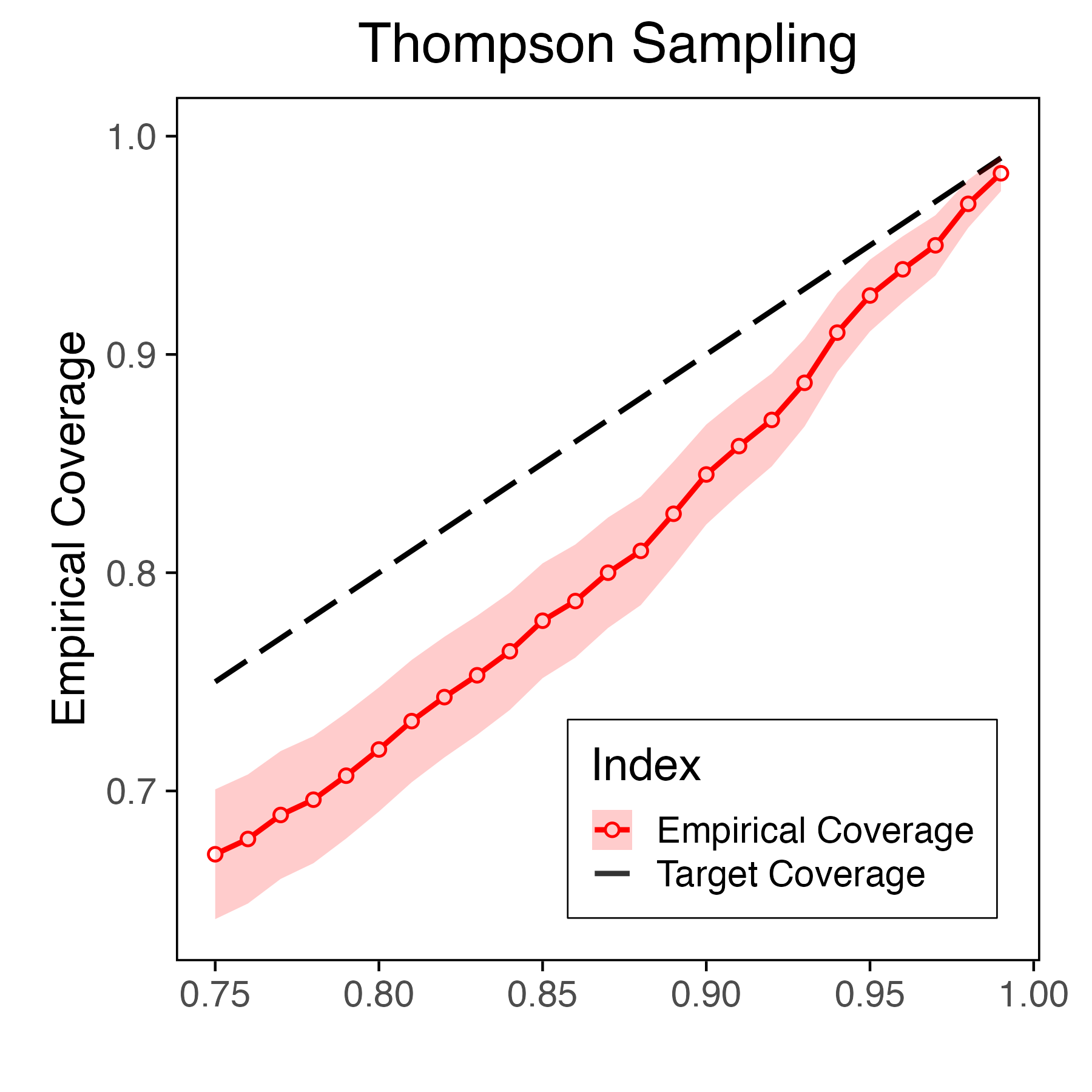}
        \caption*{\qquad Nominal Coverage}
    \end{minipage}
    \hfill
    \begin{minipage}[t]{0.295\textwidth}
        \centering
        \includegraphics[trim=0 30 0 0, clip,width=\linewidth]{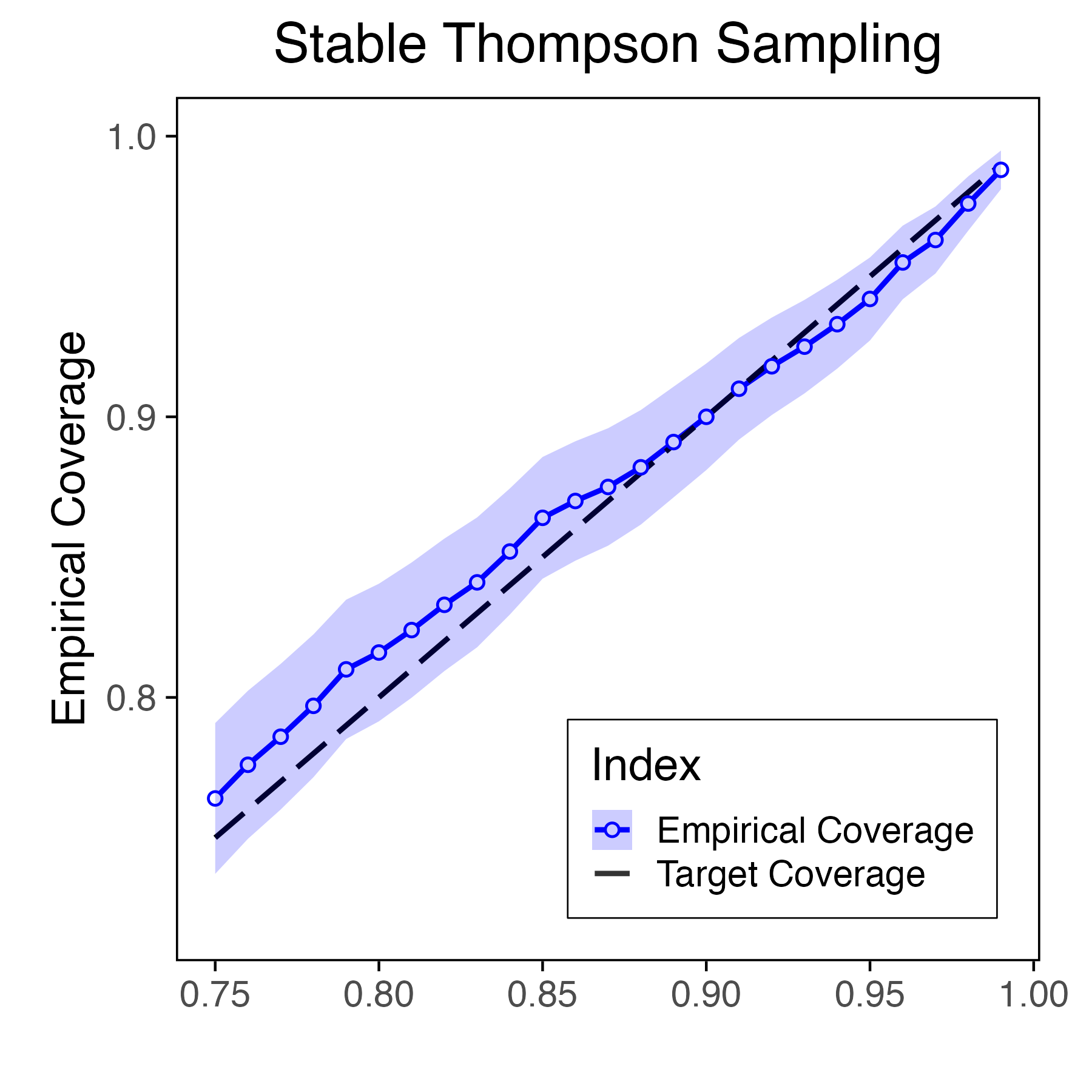}
        \caption*{\qquad Nominal Coverage}
    \end{minipage}
    \hfill
    \vspace{1mm}
    \caption{Empirical illustration of inferential failure  under Thompson Sampling and its correction via stabilization. We consider a two armed Gaussian bandit problem with equal arm mean. \textbf{Left:} Distribution of normalized sample means under standard Thompson Sampling  which deviates from Gaussianity. \textbf{Middle:} Coverage probabilities of nominal $(1 - \alpha)$ confidence intervals under Thompson Sampling, showing systematic under-coverage. \textbf{Right:} Coverage under Stable Thompson Sampling, which inflates posterior variance and restores alignment with the target coverage curve.}
    \label{fig:ts-sts-comparison}
\end{figure}

To address this issue, we propose a modification to the Thompson Sampling algorithm, which we call \emph{Stable Thompson Sampling}, that inflates the posterior variance by a logarithmic factor. This small change leads to significant theoretical benefits. In Theorem~\ref{main-thm}, we show that the modified algorithm satisfies the \textit{Lai and Wei stability condition}~\cite{laiwei82}, from which asymptotic normality of the sample means follows. Furthermore, Proposition~\ref{prop:regret-bd} establishes that the regret of Stable Thompson Sampling increases only logarithmically relative to standard Thompson Sampling. We complement our theoretical analysis with simulations (Section~\ref{subsec:sim-result}), which confirm that the stabilized variant yields valid coverage across a range of confidence levels while maintaining efficient exploration.
\vspace{-.4cm}
\section{Multi-armed bandit and Thompson sampling}
\label{sec:inference-pitfalls}
\vspace{-5pt}
We consider a Gaussian multi-armed bandit problem with $K$ arms. At each time step $t$, the player selects an arm $A_t\in[K]$ and receives a random reward $X_t$ from the distribution $\Ncal(\mu_{A_t}, \sigma^2_{A_t})$. The goal of a bandit algorithm is to select a sequence of actions $\{ A_t \}_{t = 1}^T$ such that the expected-regret over $T$ rounds, defined as,
\vspace{-10pt}
\begin{align}
    \label{eqn:regret}
    \regret(T) = T \mustar  - \Exs \left[ \sum_{t = 1}^T \X_t\right] 
\end{align}
is minimized. Here, $\mustar:= \max \limits_a \mu_a$ denotes the largest arm mean.  

Several algorithms have been proposed to address this problem, including Upper Confidence Bound (UCB) methods \cite{auer2002finite,lai1979adaptive}, which leverage optimism in the face of uncertainty, and $\varepsilon$-- greedy strategies, which balance exploration and exploitation. Among these, Thompson Sampling (TS)~\cite{thompson1933likelihood} has emerged as a particularly effective approach due to its empirical performance and strong theoretical guarantees \cite{agrawal2017near,kaufmann2012thompson}. TS operates by maintaining a posterior distribution over the arm means and sampling actions according to their probability of being optimal, seamlessly integrating exploration and exploitation. We detail the algorithm in Algorithm~\ref{algo:TS}. 
\begin{algorithm}[htbp!]
    \caption{Thompson Sampling}\label{algo:TS}
    \begin{algorithmic}[]
	\State{\textbf{Inputs}: (a) Number of epoch $T$ \\
            \textbf{Initialize}: Pull each arm $a \in [K]$ once and set $\muhat_{a,K}$ as the associated reward, and set $n_{a,K} = 1$.
            }
        \For{$t = K + 1, \ldots, T$}
            \State For each arm $a \in [K]$, sample $\sample_{a,t} \sim \Ncal\left(\muhat_{a,t-1}, \frac{1}{n_{a,t-1}}\right)$.
            \State Pull arm $A_t = \arg\max \limits_a \sample_{a,t}$ and observe associated reward $X_t$.
            \State For $a \in [K]$, update  $$ n_{a,t} = n_{a,t - 1} + \I{\{A_t=a \}} \qquad \text{and} \quad    \widehat{\mu}_{a,t} =  \frac{n_{a,t - 1}\muhat_{a,t-1} + X_t \I{\{A_t=a \}} }{n_{a,t}}$$
        \EndFor
    \end{algorithmic} 
\end{algorithm}

\vspace{-.6cm}
\section{Statistical Inference with Thompson Sampling}
\label{sec:inference}
\vspace{-5pt}
Beyond minimizing regret, there is growing interest in constructing confidence intervals (CIs) for the arm means $\{\mu_a\}_{a \in [K]}$. Such inference is critical in applications like clinical trials, where quantifying the uncertainty of treatment effects is essential~\cite{confdistrclinicialtrial}, or in A/B testing, where precise effect-size estimation informs decision-making~\cite{dubarry2015confidenceintervalsabtest,JunZheng2025ABtest}. However, standard methods for constructing CIs—such as those based on asymptotic normality of sample means—fail in the bandit setting because the data-collection process is inherently non-i.i.d. Due to adaptive sampling, the central limit theorem (CLT) does not directly apply, and naive estimators can exhibit significant bias \cite{nie18a,shin2019are}. 

To address this challenge, recent work has focused on the stability of bandit algorithms~\cite{han2024ucb,khamaru2024inference,fan2022typical,kalvitzeevi}, a notion introduced by \cite{laiwei82}.
\vspace{-5pt}
\begin{definition}
  A bandit algorithm~\algo\; is called stable if for all arm $a\in[K]$  there exist non-random scalars $\{ \nstar_{a, T}(\algo) \}_{a \in [K]}$, which depend on the problem parameters $T,\{\mu_a\}_{a\in[K]},\{\sigma^2_a\}_{a\in[K]}$,  such that
    \begin{align}
        \label{eq:stability}
        \frac{n_{a,T}(\algo)}{\nstar_{a,T}(\algo)} \inprob 1 \qquad  \text{and} \qquad  \nstar_{a,T}(\algo) \to \infty \qquad \text{as} \quad  T \to \infty. 
    \end{align}
Here, $n_{a,T}(\algo) = \sum_{t=1}^T \I{\{A_t=a\}}$ denotes the number of times arm $a$ is pulled in $T$ rounds by algorithm~\algo. 
\end{definition}
\vspace{-5pt}
At a high level, the stability of an algorithm ensures that the arm pull-behavior becomes less noisy as $T$  becomes large.  Following an argument from \cite[\textsc{Theorem 3}]{laiwei82} it is easy to show that for any stable algorithm $\algo$, the arms means $\{\muhat_{a, T}(\algo)\}_{a \in [K]}$ are asymptotically normal.  Formally, 
\begin{lemma}[\cite{laiwei82}]
\label{lem:lai-wei-lemma}
Given any stable algorithm~\algo, let $\muhat_{a,T}(\algo)$ and $\sigmahat^2_{a,T}$, respectively, denote the sample mean and variance of arm $a$ rewards at time $T$. Then for all $a \in [K]$
\begin{align}
    \label{eqn:Lai-normality}
    \frac{\sqrt{n_{a,T}(\algo) }}{\sigmahat_{a,T}} \cdot \left( \muhat_{a,T}(\algo) - \mu_a \right) \indist  \Ncal(0, 1)
\end{align}
\end{lemma}
\subsection{Instability  of Thompson sampling}
It is now natural to ask whether the Thompson Sampling Algorithm~\ref{algo:TS} satisfies the stability condition~\eqref{eq:stability}. In this section, we answer this question in the negative. In their recent paper, \cite{han2026thompsonsamplingprecisearmpull} proves that Thompson Sampling is stable for an arm if it is sub-optimal or unique optimal. For multiple optimal arms,  the vector of arm-pull proportions converges to an invariant distribution of a stochastic differential equation. We also provide a
careful simulation in a simple two-armed bandit problem that reveals the \emph{instability} of Thompson sampling.
\paragraph{A simulation study:}
We numerically study the stability properties of the Thompson Sampling Algorithm~\ref{algo:TS} for a two-armed Gaussian bandit problem in two cases:
\vspace{-1pt}
\begin{itemize}
    \item[(a)]{unique optimal arm}:  $(\mu_1, \sigma_1^2) = (1,1)$ and  $(\mu_2, \sigma_2^2) = (0,1)$.
    \vspace{-5pt}
    \item[(b)]{non-unique optimal arm}: $(\mu_1, \sigma_1^2) = (\mu_2, \sigma_2^2) =  (0,1)$. 
\end{itemize}
In both cases, we plot the histograms of the arm pull, scaled by an appropriate quantity, based on $T = 10^4$ and $10^4$ Monte Carlo simulations.
\begin{figure}[htbp]
    \centering
    \hfill
    \begin{minipage}[t]{0.46\textwidth}
        \centering
        \includegraphics[trim=0 0 0 60, clip, width=\textwidth]{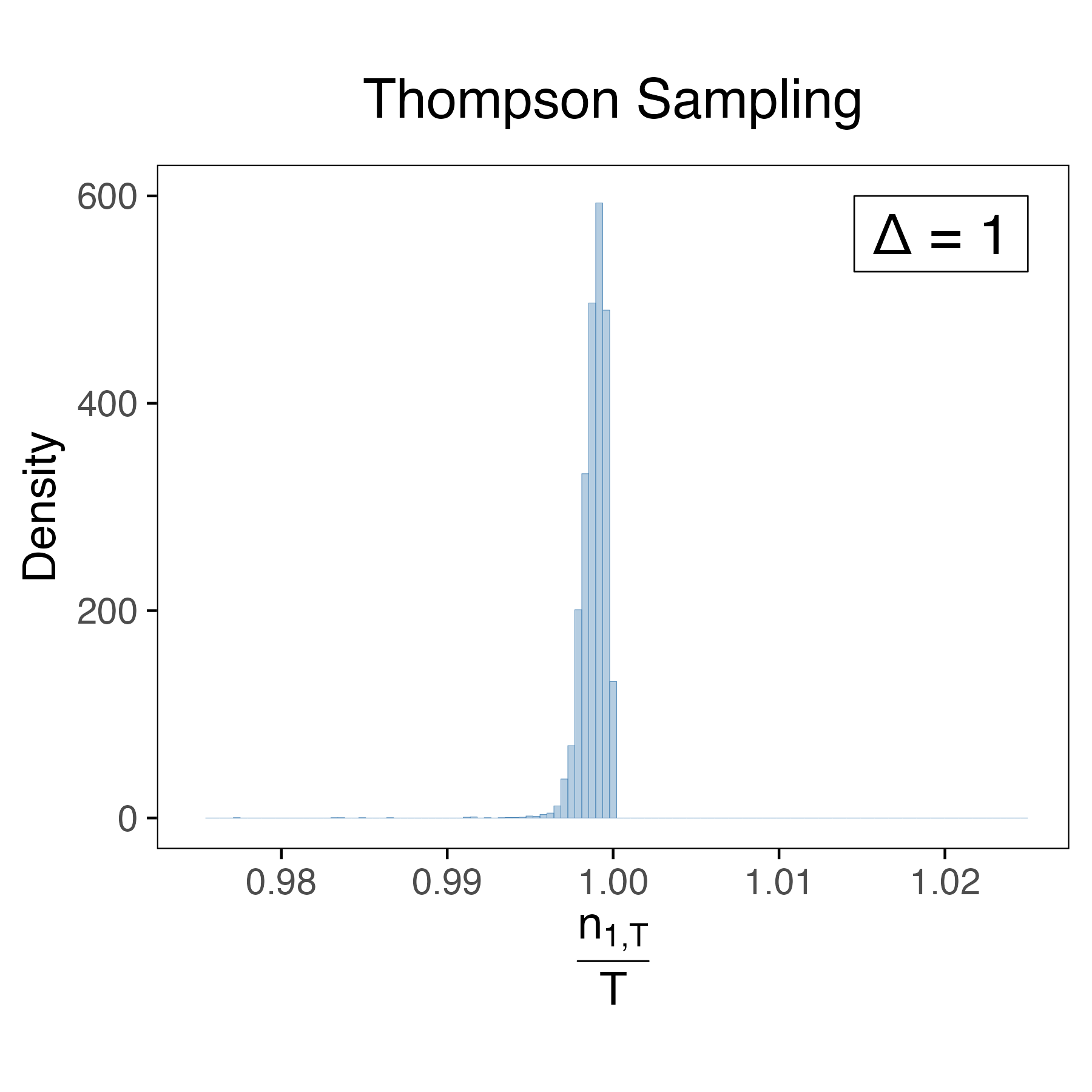}
    \end{minipage}
    \hfill
    \begin{minipage}[t]{0.46\textwidth}
        \centering
        \includegraphics[trim=0 0 0 60, clip, width=\textwidth]{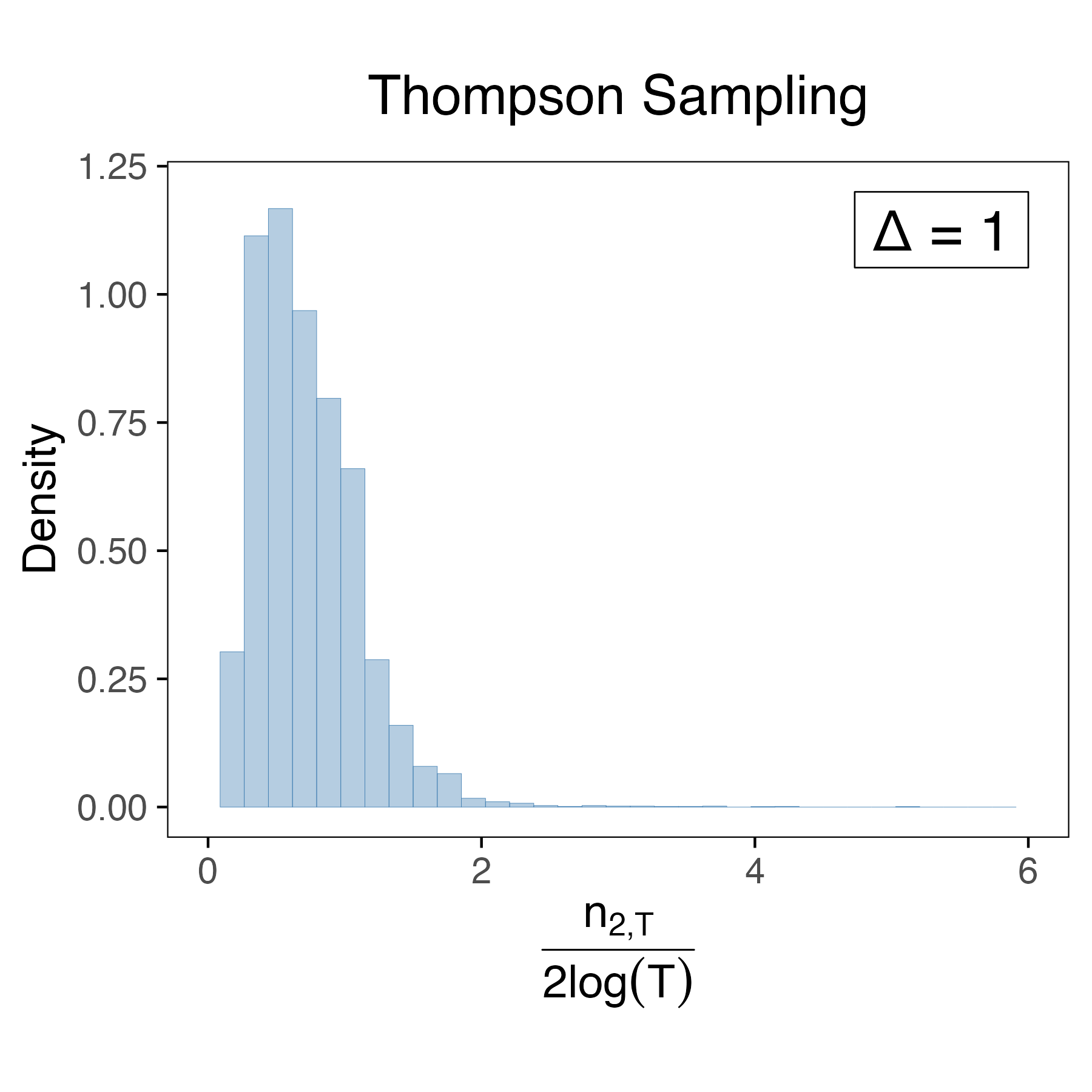}
    \end{minipage}
    \hfill
    \\
    \hfill
    \begin{minipage}[t]{0.46\textwidth}
        \centering
        \includegraphics[trim=0 20 0 60, clip, width=\textwidth]{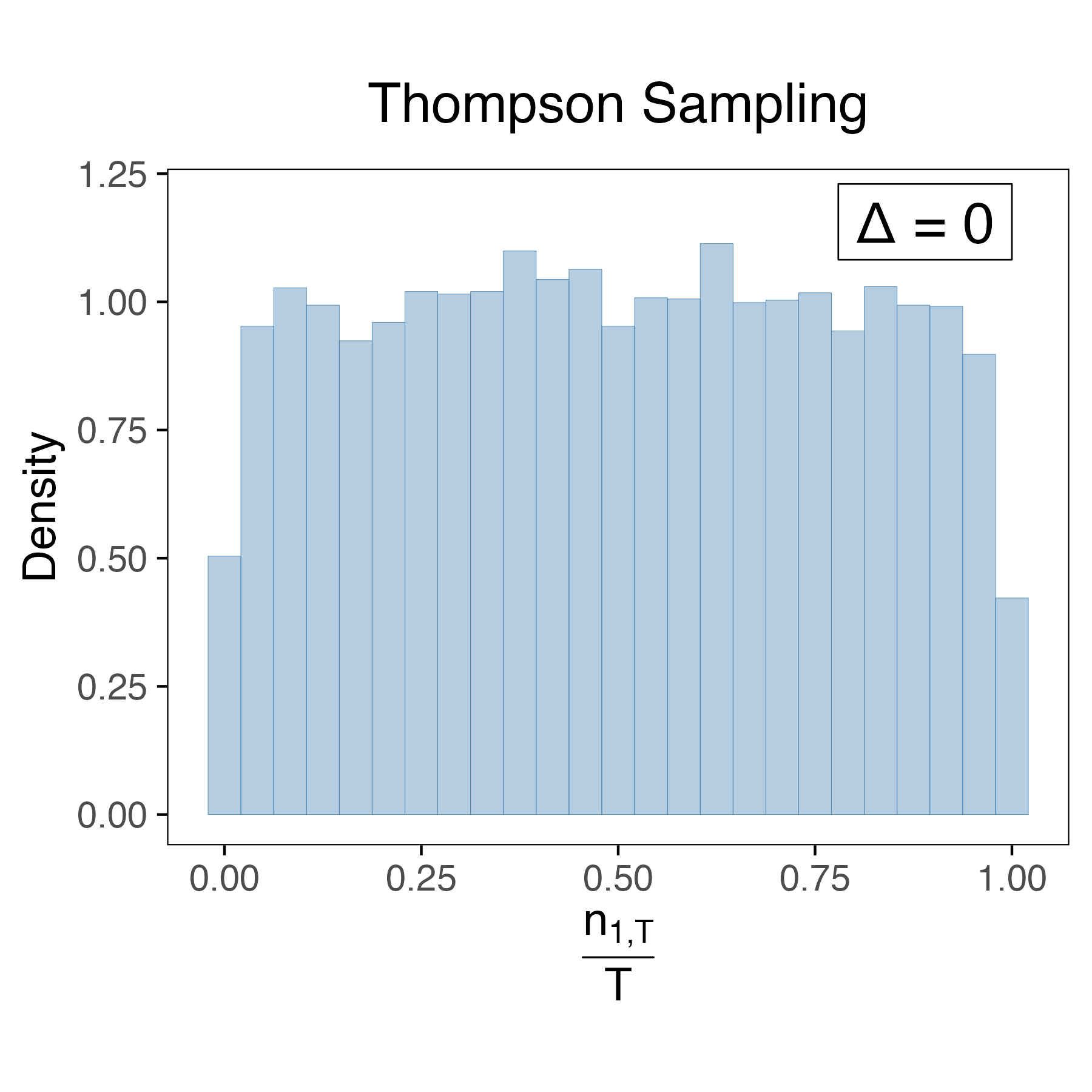}
    \end{minipage}
    \hfill
    \begin{minipage}[t]{0.46\textwidth}
        \centering
        \includegraphics[trim=0 20 0 60, clip, width=\textwidth]{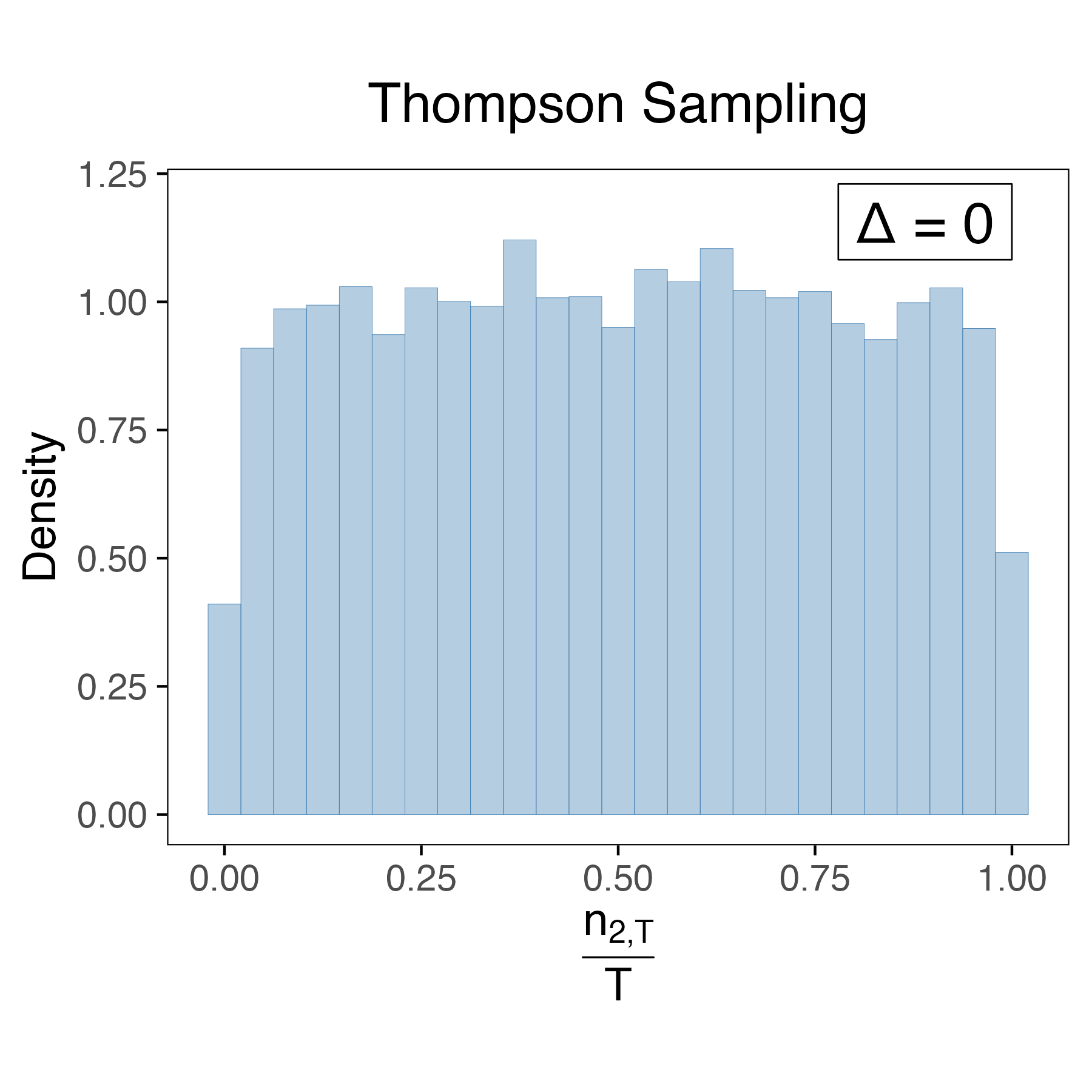}
    \end{minipage}
    \hfill
    \caption{Empirical distributions based on $10{,}000$ independent Thompson Sampling runs with $T = 10{,}000$ pulls. 
    \textbf{Top row:} Histograms of $n_{1,T}/T$ and $n_{2,T}/(2 \log T)$ under unequal means ($\Delta = 1$). These plots illustrate that Thompson Sampling satisfies the stability condition~\eqref{eq:stability} when $\Delta = 1$, but the condition fails to hold when $\Delta = 0$. \textbf{Bottom row:} Histograms of $n_{1,T}/T$ and $n_{2,T}/T$ under equal mean rewards ($\Delta = 0$).}
    \label{fig:ts-combined}
\end{figure}

The top row in Figure~\ref{fig:ts-combined} illustrates the behavior of the arm pulls for the unique optimal arm, corresponding to case (a). The plot shows that the histogram of $\frac{n_{1,T}}{T}$ and the histogram of $\frac{n_{2,T}}{2 \log T}$ both concentrate around 1, thereby providing numerical evidence that the stability condition~\eqref{eq:stability} is satisfied in this setting. Indeed, for a two-armed bandit problem where arm 1 is the unique optimal arm, the sublinear regret of Thompson Sampling~\cite{agrawal2017near} immediately yields that  $\frac{n_{1,T}}{T} \inprob 1$. Recent works~\cite{fan2022typical,han2026thompsonsamplingprecisearmpull} show that the asymptotic number of pulls of a suboptimal arm matches the celebrated \cite{lairobbins} lower bound. In the context of a two-armed bandit problem considered in our simulation, this result yields
\begin{align*}
    \frac{n_{2,T}}{2 \log T} \; \inprob \;  \frac{1}{\Delta^2}  
\end{align*}
where $\Delta:=\mu_1 - \mu_2$ is arm mean gap.

The bottom row in Figure~\ref{fig:ts-combined} demonstrates the behavior of arm pulls when both arms are identical, corresponding to case (b). Unfortunately, neither the regret bounds --- which is zero in this case for any algorithm --- nor the results of \cite{fan2022typical} shed any light on the arm pull behavior. The plots suggest that the arm pull ratio $\frac{n_{1,T}}{T}$ and $\frac{n_{2,T}}{T}$ \emph{do not} concentrate, and empirical distribution of the arm pulls resembles a uniform distribution on $[0,1]$. In a recent work, \cite{kalvitzeevi} studied the stability properties of the Thompson Sampling in a very special case of two-armed Bernoulli bandit problem with equal arm means. They show that the arm pull ratios $\frac{n_{1,T}}{T}$ and $\frac{n_{2,T}}{T}$ converge in distribution to a uniform distribution on $[0,1]$. Additionally, \cite{han2026thompsonsamplingprecisearmpull} shows instability of a generalized version of Thompson Sampling in presence of multiple optimal arms.
\section{Stable Thompson Sampling}
\label{sec:stable-TS-algorithm}
In this section, we propose and analyze a modification of Thompson Sampling, which we refer to as Stable Thompson Sampling. Our approach is motivated by a recently established stability property of the Upper Confidence Bound (UCB) algorithm. A series of recent works~\cite{fan2022typical,kalvitzeevi,han2024ucb} attribute this stability to UCB’s characteristic bonus term, $\sqrt{\frac{2 \log T}{n_{a, t - 1}}}$, which inflates reward estimates to account for statistical uncertainty in estimating the sample mean. The UCB estimate for arm $a$ at time $t$ takes the form
\begin{align}
    \label{eq:ucb-update}
    \text{UCB}_{a,t} = \widehat{\mu}_{a,t - 1} + \sqrt{\frac{2 \log T}{n_{a,t -1}}}
\end{align}
where $\widehat{\mu}_{a, t-1}$ is the empirical mean, $n_{a, t - 1}$ is the number of times arm $a$ has been pulled and $T$ is the total number of rounds. The bonus term exceeds the typical deviation of the sample mean, which is only of the order $\sqrt{\frac{2\log\log T}{n_{a, t - 1}}}$ under sub-Gaussian noise, as implied by the Law of iterated logarithm~\cite[\textsc{Theorem $8.5.2$}]{durrett2019probability}.  This deliberate inflation guards against underestimation of the true mean and contributes to stability (cf. definition~\eqref{eq:stability}). Motivated by this insight, we incorporate a similar safeguard into Thompson Sampling:
\begin{quote}
\emph{We inflate the posterior variance in Thompson Sampling by a factor $\gamma_T \gg 2\log\log T$ to guard against uncertainty in estimating the arm-means.}
\end{quote}
\noindent We refer to this variant as \emph{Stable Thompson sampling}; the full algorithm is presented in Algorithm~\ref{alg:Modified-Thompson-sampling}.

\begin{algorithm}
    \caption{Stable Thompson Sampling}\label{alg:Modified-Thompson-sampling}
    \begin{algorithmic}[]
	\State{\textbf{Inputs}: (a) Number of epoch $T \;\;$  (b) Tuning parameter $\gamma_T$ \\
            \textbf{Initialize}: Pull each arm $a \in [K]$ once and set $\muhat_{a,K}$ as the associated reward. Set $n_{a,K} = 1$.
            }
        \For{$t = K + 1, \ldots, T$}
            \State For each arm $a \in [K]$, sample $\sample_{a,t} \sim \Ncal\left(\muhat_{a,t-1}, \frac{\gamma_T}{n_{a,t-1}}\right)$.
            \State Pull arm $A_t = \arg\max_a \sample_{a,t}$ and observe associated reward $X_t$.
            \State Set
            \vspace{-.3cm}
            $$ n_{a,t} = n_{a,t - 1} + \I{\{A_t=a \}} \qquad \text{and} \quad    \widehat{\mu}_{a,t} =  \frac{(n_{a,t}-1)\muhat_{a,t-1} + X_t \I{\{A_t=a \}} }{n_{a,t}}$$
        \EndFor
    \end{algorithmic} 
\end{algorithm}

\subsection{Theoretical guarantees}
In this section, we state our main theoretical guarantees on Stable Thompson Sampling. We show that Stable Thompson algorithm is stable for a $K$-- armed bandit problem under some restrictions. Throughout, we assume that the tuning parameter $\gamma_T$ in Algorithm~\ref{alg:Modified-Thompson-sampling} satisfies
\begin{align}
    \label{gamma-cond}
    \frac{\gamma_T}{\log\log T} \rightarrow \infty 
    \qquad \text{and} \qquad  \frac{\sqrt{\log T}}{\gamma_T} \rightarrow \infty 
    \quad \text{as} \;\; T \rightarrow \infty. 
\end{align}
\noindent Under this setup, we have:\footnote{The initial version of this article, which was arXiv-ed on May 2025, dealt with the $2$-- armed scenario. This updated version was submitted to 39th Annual Conference on Learning Theory (COLT 2026) on February 4, 2026. After our manuscript submission, we later became aware of the work by \textit{Yan and Zhong} \cite{yan2026optimismstabilizesthompsonsampling}, arXiv-ed on February 5, 2026, which resolved the general $K$-- armed Gaussian bandit problem.}
\begin{theorem}[Stability in $K$-- armed bandit]
    \label{main-thm}
    In a $K$-- armed Gaussian bandit problem, let $\mathcal{I}:=\{j\in[K]:\Delta_j=0\}$. Let $|\mathcal{I}|\le2$ and we run Algorithm~\ref{alg:Modified-Thompson-sampling} with tuning parameter $\gamma_T$ satisfying condition~\eqref{gamma-cond}. Then, Algorithm~\ref{alg:Modified-Thompson-sampling} is stable. Concretely, the following hold
    \begin{align*}
        \frac{n_{j,T}}{T}\inprob\frac{1}{|\mathcal{I}|}\;\forevery j\in\mathcal{I}\qquad\text{and}\qquad\frac{n_{a,T}}{\gamma_T\cdot\log T}\inprob\frac{2}{\Delta_a^2}\;\forevery a\notin\mathcal{I}
    \end{align*}
\end{theorem}
\noindent We prove a special case of Theorem~\ref{main-thm} in Section~\ref{sec:outline-main-thm} and defer the proof of the general case to Appendix~\ref{sec:proof-main-thm-app}. An immediate consequence of this theorem, together with Lemma~\ref{lem:lai-wei-lemma}, is that sample mean of each arm is asymptotically normal. As a result, given any target level $\alpha > 0$, we have the following asymptotically exact confidence intervals. Specifically,  for arms $a \in[K]$ we have $\lim_{T \rightarrow \infty } \prob \left( \CI_{a, \alpha, T} \ni \mu_a \right) = 1 - \alpha$ where
\begin{align}
    \label{eq:asymp-CI}
    \CI_{a, \alpha, T} = \left[ \widehat{\mu}_{a,T} - z_{1 - \alpha/2} \frac{\sigmahat_{a, T} }{\sqrt{n_{a,T}}}, \quad \widehat{\mu}_{a,T} + z_{1 - \alpha/2} \frac{\sigmahat_{a, T} }{\sqrt{n_{a,T}}} \right]
\end{align}
Here, $\widehat{\mu}_{a,T}$ and $\sigmahat_{a,T}$ are the sample mean and sample standard deviation of the rewards from arm $a$, respectively. 
\paragraph{Regret comparison:}
While Stable Thompson Sampling enables valid inference despite adaptive data collection, it is natural to ask how much we lose in terms of expected regret. The following regret bound addresses this question in the general $K$-- armed bandit setting.
\begin{proposition}
\label{prop:regret-bd}
For $K$-- armed bandit problem, the regret of Algorithm~\ref{alg:Modified-Thompson-sampling} is upper bounded by
\begin{align}
    \regret(T) \le \sum_{j:\Delta_j > 0} \left[ \kappa_1 \cdot {\frac{\gamma_T}{\Delta_j} \cdot \log\left( \frac{T\Delta_j^2}{\gamma_T} + \kappa_2 \right)} + {\kappa_3 \cdot\frac{\gamma_T}{\Delta_j}} \right],
\end{align}
where \(\Delta_j = \mustar - \mu_j\), and \(\kappa_1, \kappa_2, \kappa_3\) are positive universal constants.
\end{proposition}

Comparing this bound with the regret guarantee for standard Thompson Sampling~\cite[\textsc{Theorem~2}]{korda2013thompson}, we see that the regret of the Stable Thompson Sampling algorithm is at most a factor of \(\gamma_T\) worse, up to universal constants.
\vspace{-5pt}
\subsection{Illustrative simulations}
\label{subsec:sim-result}
\vspace{-5pt}
We now present numerical simulation results for Algorithm~\ref{alg:Modified-Thompson-sampling} on a four-armed Gaussian bandit problem with $\mu=(1,1,0.5,0)$, using $T = 10^4$ rounds and $\gamma_T = 4 \cdot (\log T)^{0.4}$, averaged over $10^4$ independent experiments.

\begin{figure}[htbp!]
    \centering
    \hfill
    \begin{minipage}[t]{0.24\textwidth}
        \centering
        \includegraphics[trim=0 0 0 22, clip,width=\linewidth]{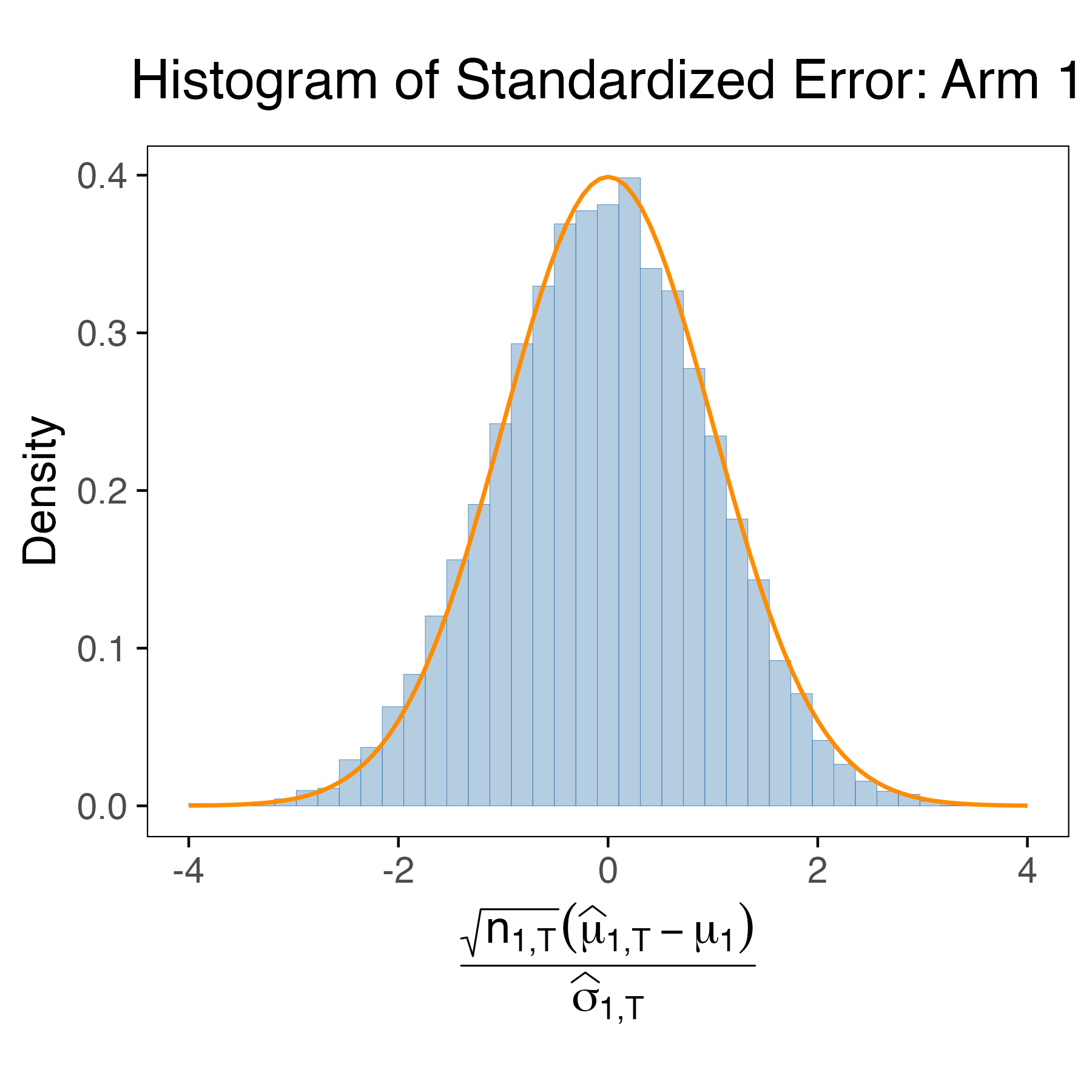}
    \end{minipage}
    \hfill
    \begin{minipage}[t]{0.24\textwidth}
        \centering
        \includegraphics[trim=0 0 0 22, clip,width=\linewidth]{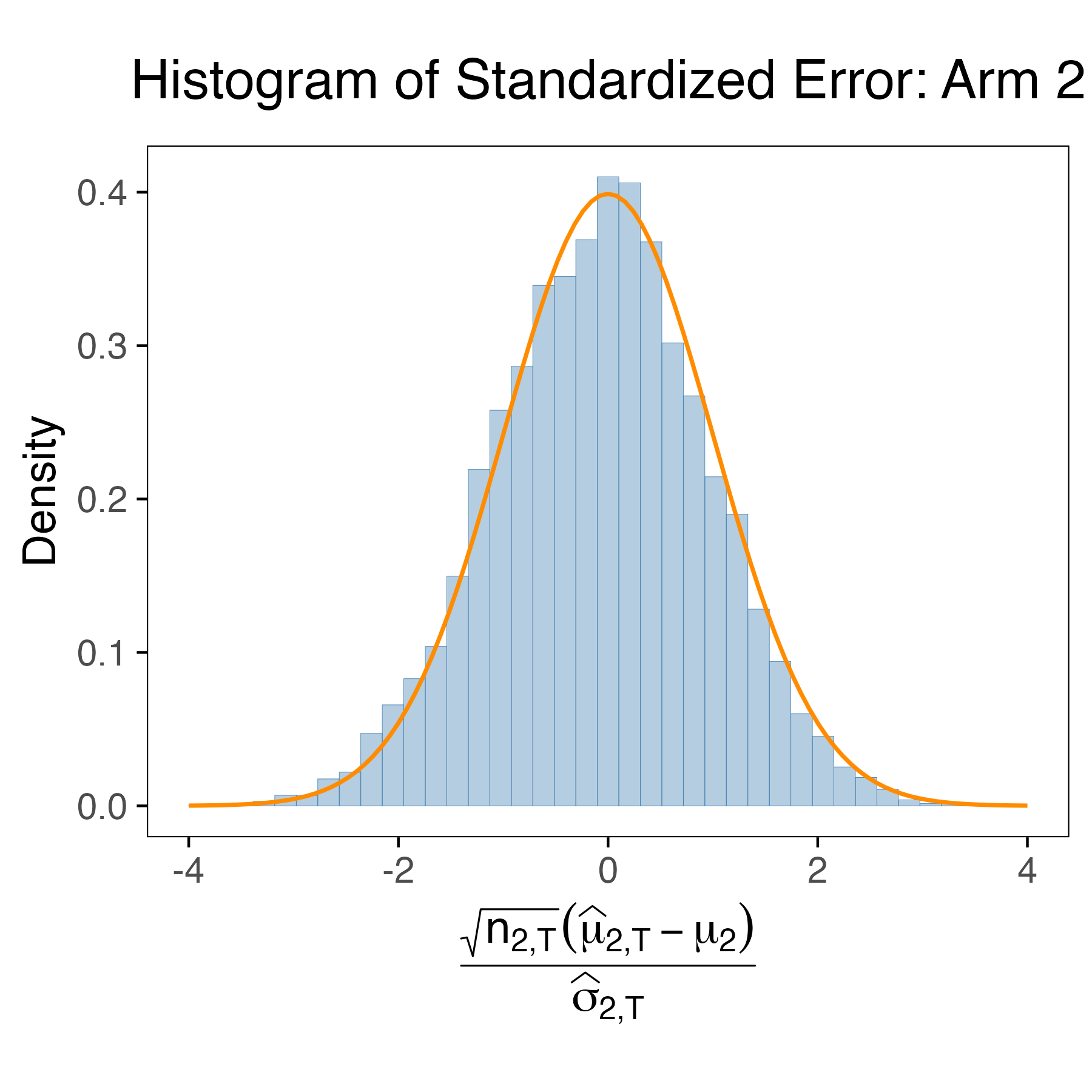}
    \end{minipage}
    \begin{minipage}[t]{0.24\textwidth}
        \centering
        \includegraphics[trim=0 0 0 22, clip,width=\linewidth]{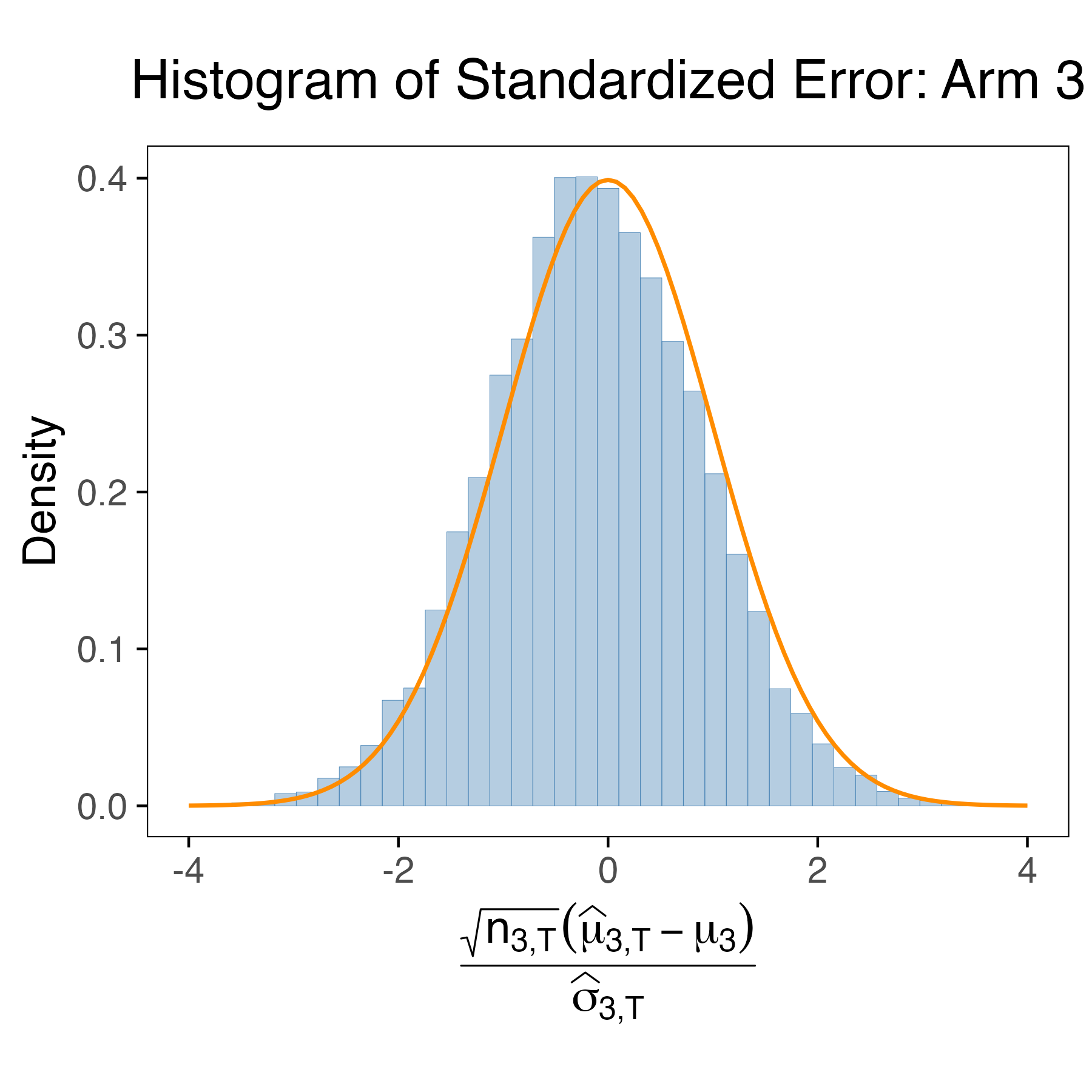}
    \end{minipage}
    \hfill
    \begin{minipage}[t]{0.24\textwidth}
        \centering
        \includegraphics[trim=0 0 0 22, clip,width=\linewidth]{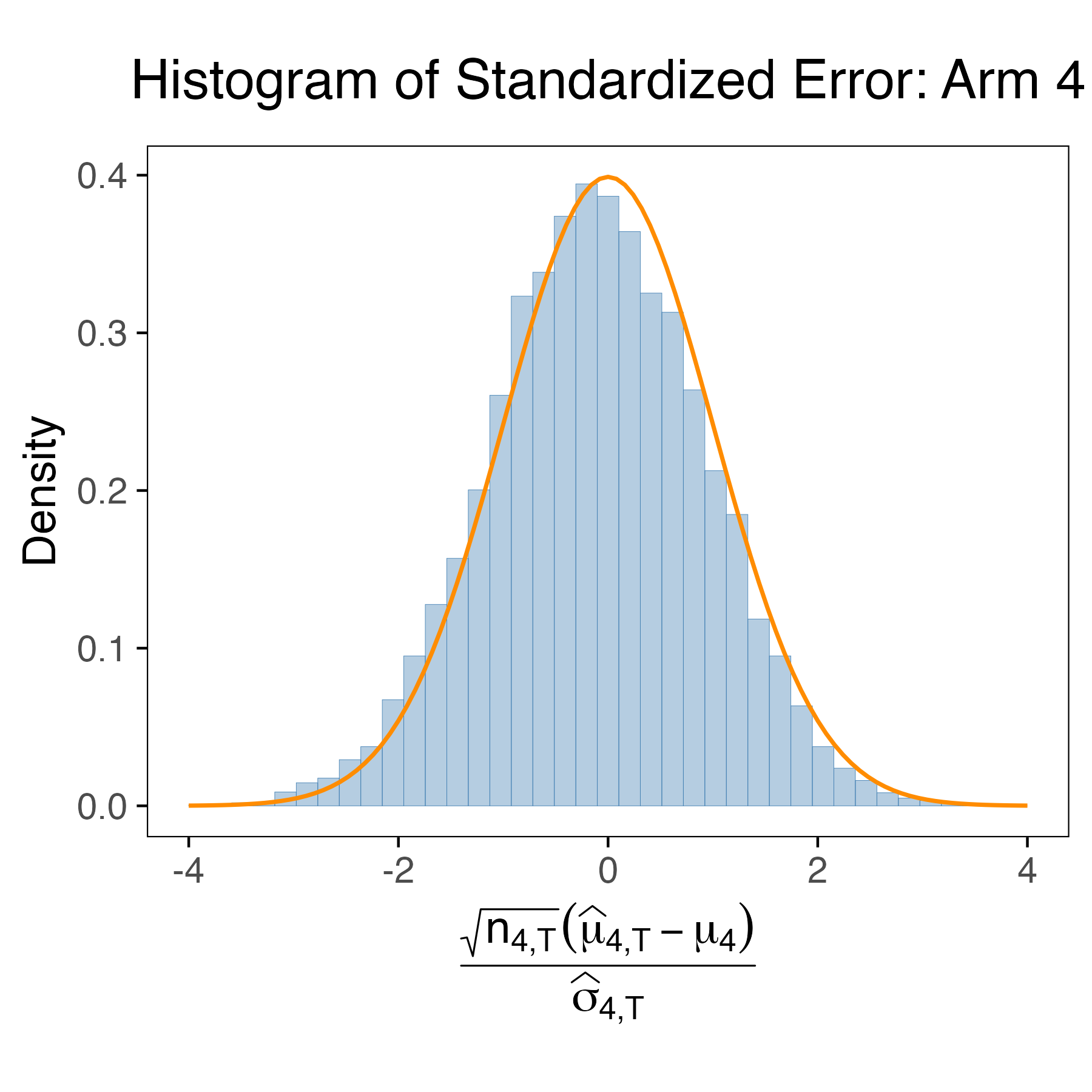}
    \end{minipage}
    \hfill
    \caption{Empirical behavior of Stable Thompson Sampling for a four-armed bandit with $\mu=(1,1,0.5,0)$: standardized estimation errors \(\sqrt{n_{a,T}}(\widehat{\mu}_{a,T} - \mu_a)/\widehat{\sigma}_{a,T}\) are approximately standard normal.}
    \label{fig:sts-normality}
\end{figure}

Figure~\ref{fig:sts-normality} plots the standardized estimation errors against the standard Gaussian density, confirming the asymptotic normality of the sample means. This validates the use of normality based confidence intervals for the arm means under Stable Thompson Sampling.

\begin{figure*}[htbp!]
    \centering
    \hfill
    \begin{minipage}[t]{0.24\textwidth}
        \includegraphics[width=\textwidth]{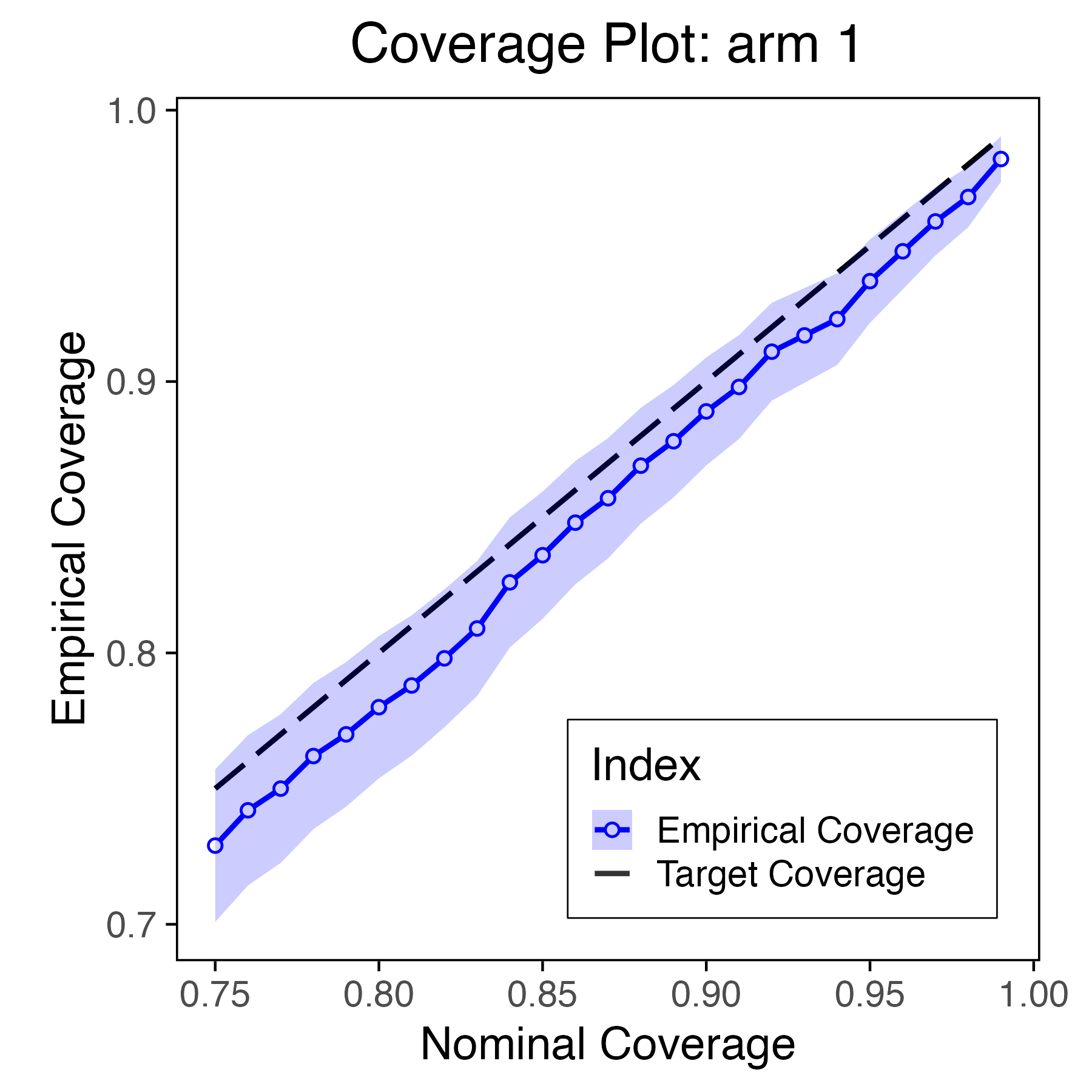}
    \end{minipage}
    \hfill
    \begin{minipage}[t]{0.24\textwidth}
        \includegraphics[width=\textwidth]{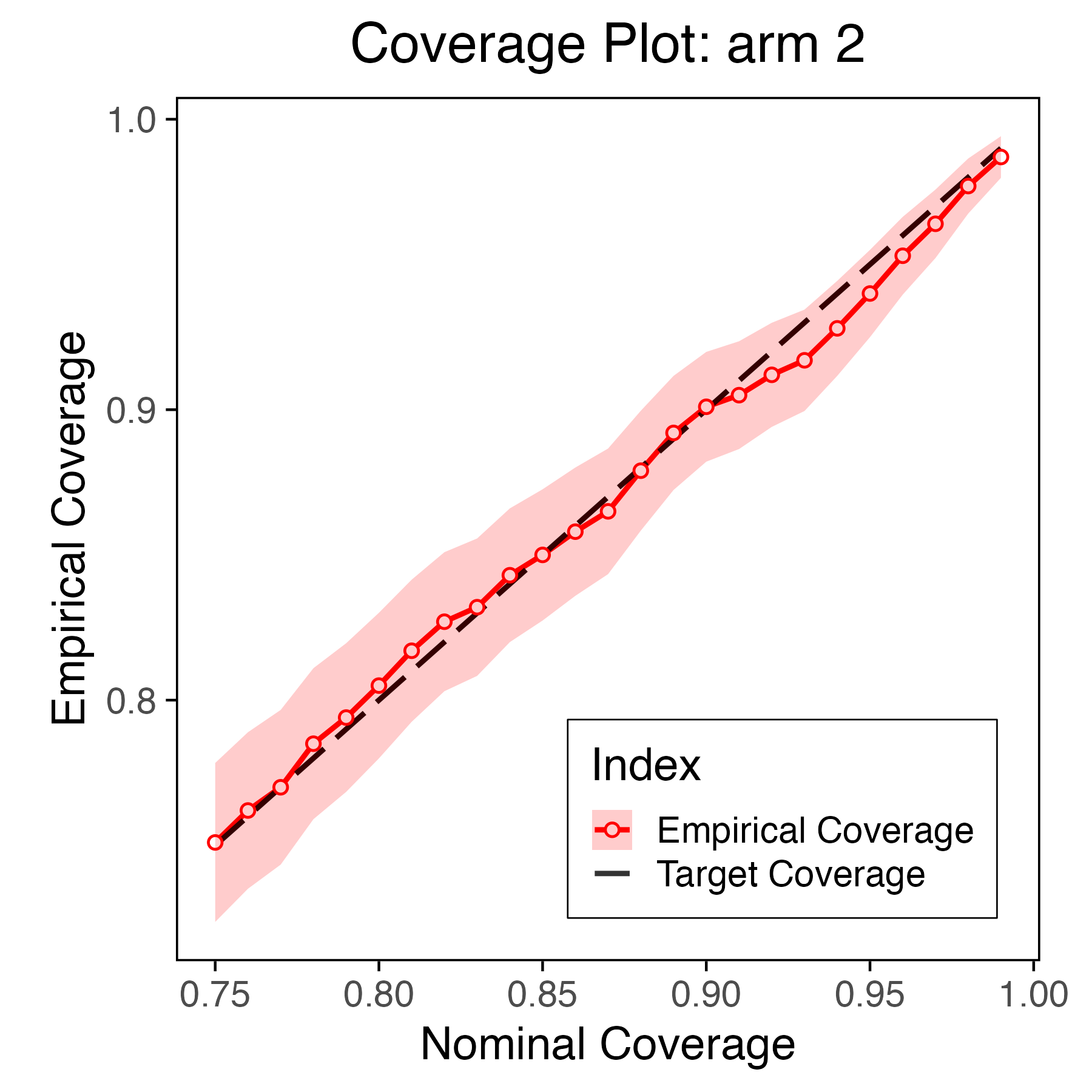}
    \end{minipage}
    \hfill
    \begin{minipage}[t]{0.24\textwidth}
        \includegraphics[width=\textwidth]{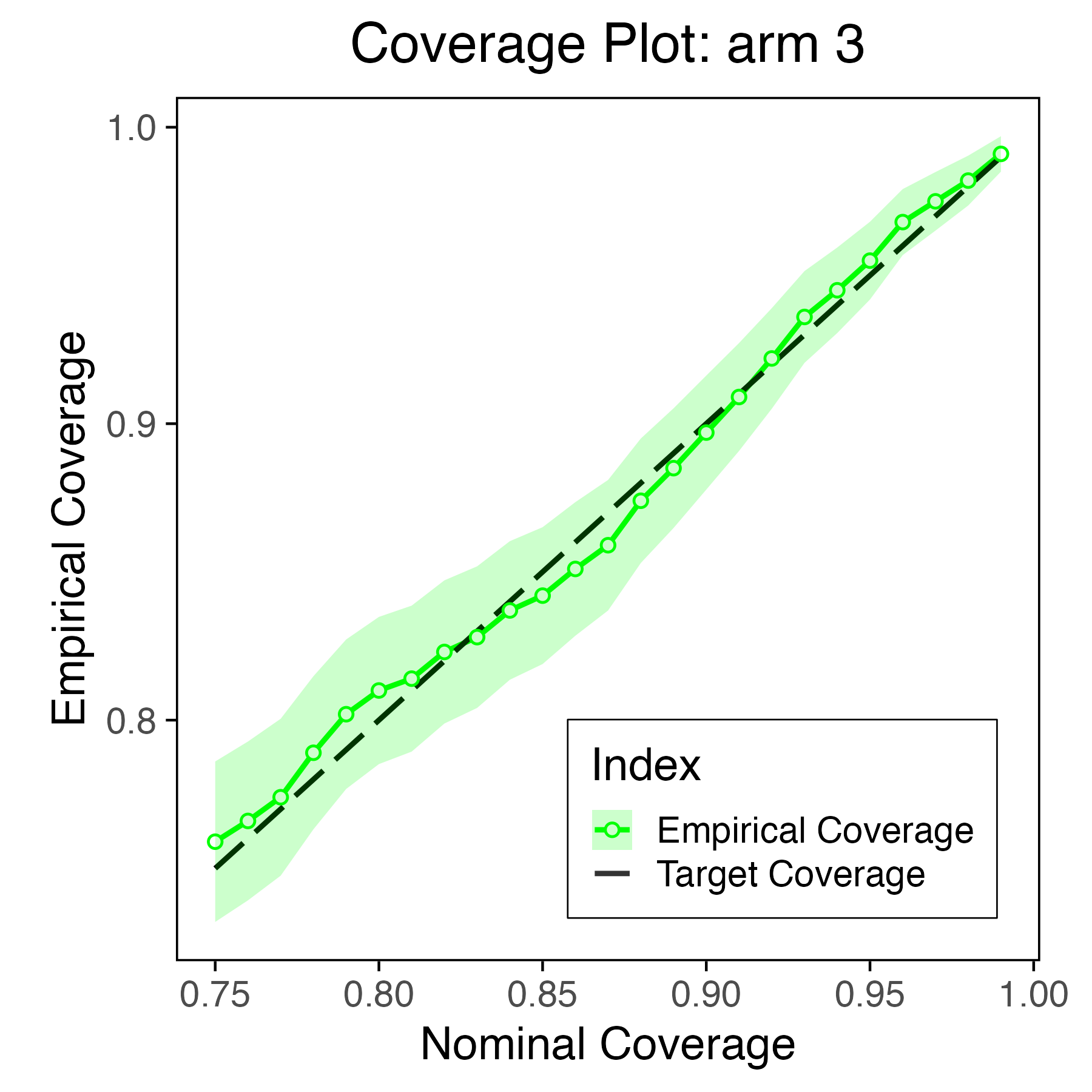}
    \end{minipage}
    \hfill
    \begin{minipage}[t]{0.24\textwidth}
        \includegraphics[width=\textwidth]{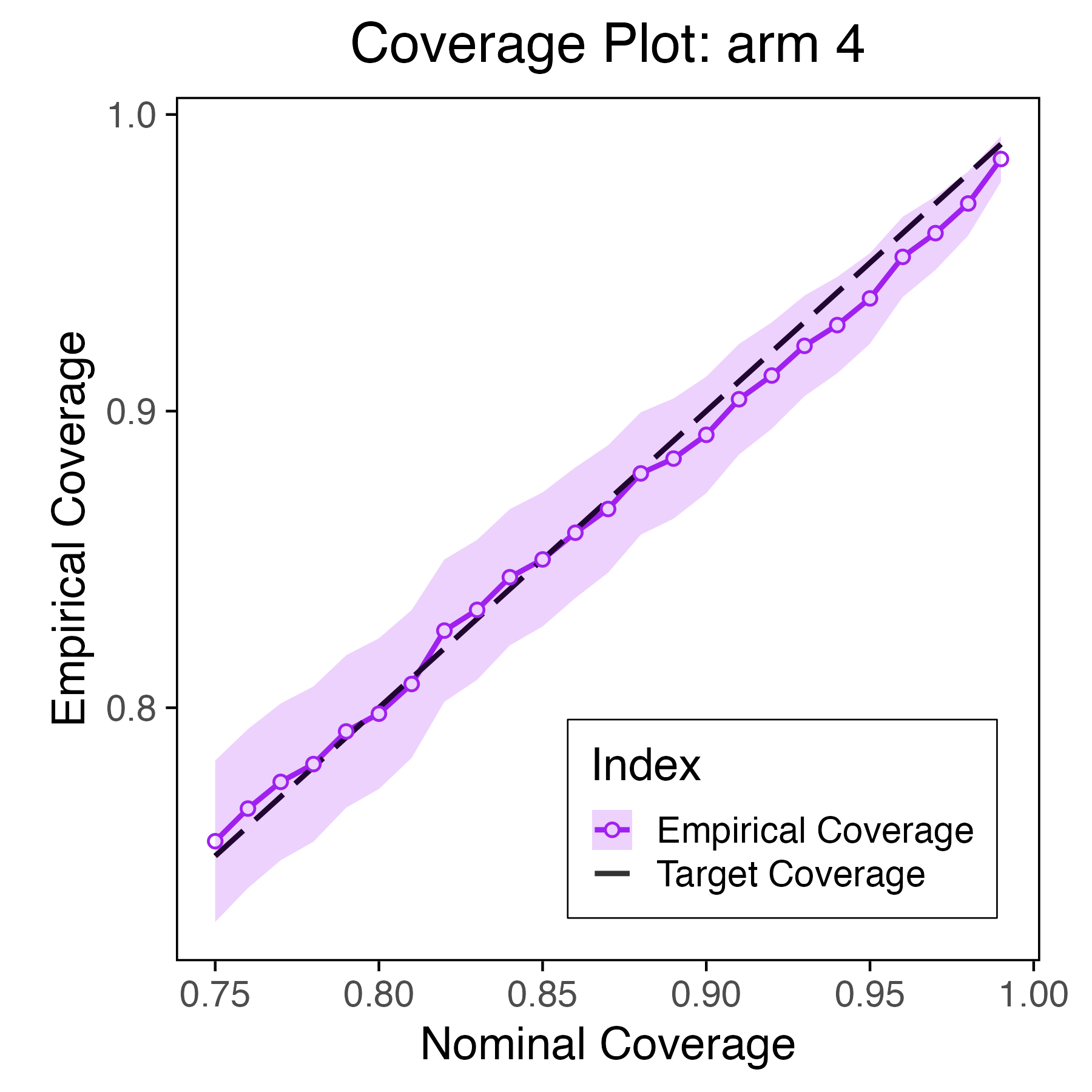}
    \end{minipage}
    \hfill
    \caption{
Empirical coverage probability of confidence intervals constructed using Stable Thompson Sampling, plotted against nominal confidence levels for each arm in a four-armed Gaussian bandit. Results are aggregated over $10,000$ independent trials, and shaded regions represent standard errors. The near alignment with the diagonal confirms the asymptotic validity of our inference procedure.
}
    \label{fig:cov-prob-sts}
\end{figure*}

We conclude the simulation section with coverage probability plots for each arm in Figure~\ref{fig:cov-prob-sts}. We evaluate a sequence of nominal confidence levels $(1 - \alpha)$ ranging from $0.75$ to $0.99$. The solid curves represent the empirical coverage probabilities for arms $1-4$, respectively, and the shaded regions denote the standard error across $10,000$ repetitions. The dashed black line marks the target nominal coverage level. Across the entire confidence range $(0.75, 0.99)$, each arm exhibits excellent agreement with the nominal levels, confirming the asymptotic validity of the confidence intervals constructed under Stable Thompson Sampling.

\section{Outline of Proof of Theorem~\ref{main-thm}}
\label{sec:outline-main-thm}
In this section, we provide a roadmap of the proof of Theorem~\ref{main-thm} for two-armed gaussian bandit. Throughout, we assume arm $1$ is optimal and the reward variance $\sigma^2=1$.
\subsection{Case 1: \texorpdfstring{$\Delta = 0$}{Delta = 0}}
\label{sec:Delta-zero}
Let us define $\UU_{t}:=\I\{A_t=1\}$ and $p_1\left(n_{1,t-1}\; ,t\right):=\E\left[\UU_{t}\; \big|\; \F_{t-1}\right]=\prob\left(A_{t}=1\; \big|\; \F_{t-1}\right)$.  For convenience, we denote $p_1\left(n_{1,t-1}\; ,t\right)$ as simply $p_{1,t}$. Rewrite $\frac{n_{1,T}}{T}=\frac{1+\sum_{t=3}^T \UU_t}{T}=\frac{1}{T}+\mathfrak{G}_1^{(T)}+\mathfrak{G}_2^{(T)}$ where $\mathfrak{G}_1^{(T)}=\frac{\sum_{t=3}^T \left(\UU_t-p_{1,t}\right)}{T}$ and $\mathfrak{G}_2^{(T)}=\frac{\sum_{t=3}^T p_{1,t}}{T}$. We prove that $\mathfrak{G}_1^{(T)}\inprob0$ and $\mathfrak{G}_2^{(T)}\inprob\frac{1}{2}$.

\noindent By Markov Inequality, for any $\varepsilon>0$, we have $\prob\left(\left|\frac{\sum_{t=3}^T \left(_t-p_{1,t}\right)}{T}\right|>\varepsilon\right)\le\frac{\E\left[\left[\sum_{t=3}^T \left(\UU_t-p_{1,t}\right)\right]^2\right]}{\varepsilon^2 T^2}$. Rewrite $\left[\sum_{t=3}^T \left(\UU_t-p_{1,t}\right)\right]^2=\left(\sum_{t=3}^T(\UU_t-p_{1,t})^2 \right)+2\left(\sum_{3\le i< j\le T}\left(\UU_i-p_{1,i}\right)\left(\UU_j-p_{1,j}\right)\right)$. Observe that, for $i<j$, $\E\left[\left(\UU_i-p_{1,i}\right)\left(\UU_j-p_{1,j}\right)\right]=\E\left[\left(\UU_i-p_{1,i}\right)\E\left[\left(\UU_j-p_{1,j}\right)\Big|\F_{j-1}\right]\right]=0$. This is because $p_{1,j}$ is $\F_{j-1}$-- measurable and hence, $\E\left[\left(\UU_j-p_{1,j}\right)\Big|\F_{j-1}\right]=\E\left[\UU_j\Big|\F_{j-1}\right]-p_{1,j}=0$. Therefore, $\prob\left(\left|\frac{\sum_{t=3}^T \left(\UU_t-p_{1,t}\right)}{T}\right|>\varepsilon\right)\le\frac{T-2}{\varepsilon^2 T^2}$. This proves $\mathfrak{G}_1^{(T)}\inprob0$.

\noindent Next, we prove $\mathfrak{G}^{(T)}_2\inprob1/2$ by showing that $p_{1,t}\inprob1/2$. Recalling the posterior sample at time $t$ is $\theta_{a,t}=\widehat{\mu}_{a,t-1}+\sqrt{\frac{\gamma_{T}}{n_{a,t-1}}}Z_{a,t}$ for $Z_{a,t}\overset{\text{i.i.d.}}{\sim}\NN(0,1)$ and letting $\widehat{\Delta}_t:=\widehat{\mu}_{1,t}-\widehat{\mu}_{2,t}$, we have
\begin{align*}
    p_{1,t}&=\prob\left(\widehat{\mu}_{1,t-1}+\sqrt{\frac{\gamma_{T}}{n_{1,t-1}}}Z_{1,t}>\widehat{\mu}_{2,t-1}+\sqrt{\frac{\gamma_{T}}{n_{2,t-1}}}Z_{2,t}\; \Big|\; \F_{t-1}\right)\\
    &=\prob\left(\sqrt{\frac{\gamma_{T}}{n_{2,t-1}}}Z_{2,t}-\sqrt{\frac{\gamma_{T}}{n_{1,t-1}}}Z_{1,t}<\widehat{\Delta}_{t-1}\; \Big|\; \F_{t-1}\right)\\
    &=\Phi\left(\widehat{\Delta}_{t-1}\sqrt{\frac{n_{1,t-1}\; n_{2,t-1}}{t\; \gamma_{T}}}\right)
\end{align*}
Now, consider the event $\EE^{(T)}:=\left\{\left|\widehat{\mu}_{j,t}-\mu_j\right|\le\sqrt{\frac{7\log(\log(T))}{n_{j,t}}}\; \forevery\; t\in[T],\; \forevery\; j=1,2\right\}$. On $\EE^{(T)}$, for all $t\in[T]$, $\left|\widehat{\Delta}_{t-1}\right|\sqrt{\frac{n_{1,t-1}\; n_{2,t-1}}{t\; \gamma_{T}}}\le\Big[\left|\widehat{\mu}_{1,t-1}-\mu_1\right|+\left|\widehat{\mu}_{2,t-1}-\mu_2\right|\Big]\sqrt{\frac{n_{1,t-1}\; n_{2,t-1}}{t\; \gamma_{T}}}\le \sqrt{7\log(\log(T))} \left(\frac{\sqrt{n_{1,t-1}}+\sqrt{n_{2,t-1}}}{t\; \gamma_T}\right)\le\sqrt{\frac{14\log(\log(T))}{\gamma_T}}$. Therefore, on the event $\EE^{(T)}$, for all $t\in[T]$, we have
\begin{align*}
    \Phi\left(-\sqrt{\frac{14\log(\log(T))}{\gamma_T}}\right)\le p_{1,t}\le \Phi\left(\sqrt{\frac{14\log(\log(T))}{\gamma_T}}\right)
\end{align*}
and hence
\begin{align}
    \label{eq:p2-bdd-delta=0}
    \frac{T-2}{T}\cdot\Phi\left(-\sqrt{\frac{14\log(\log(T))}{\gamma_T}}\right)\le \frac{\sum_{t=3}^T p_{1,t}}{T}\le\frac{T-2}{T}\cdot \Phi\left(\sqrt{\frac{14\log(\log(T))}{\gamma_T}}\right)
\end{align}
Additionally, for $t\in[T]$ and large $T$, $3\log\left(\log\left(2n_{j,t}\right)\right)+ 3\log\left(\log\left(T\right)\right) \le 7\log(\log(T))$. For $\EE_4^{(T)}$ defined in equation~\eqref{eqn:hp-events}, Lemma \ref{lem:sets-whp} yields $\prob\left(\EE^{(T)}\right) \ge\prob\left(\EE_4^{(T)}\right)\ge1-\frac{2}{\log(T)}$. Putting together the pieces, we conclude $\mathfrak{G}^{(T)}_2\inprob1/2$.
\subsection{Case 2:  \texorpdfstring{$\Delta > 0$}{Delta > 0}}
\label{sec:Delta-positive-main}
The key idea is to investigate the behavior of the waiting times between successive pulls of the sub-optimal arm. To do this, let $\TT_{2,k}$ denote the time of playing arm $2$ for the $k^{\text{th}}$ time with $\TT_{2,1}=1$. Let $\tau_2(0)=1$ and for $j\ge1$, define $\tau_2(j):=\TT_{2,j+1}-\TT_{2,j}$.\\
After we have reached round $T$, we continue to pull the arms following Algorithm~\ref{alg:Modified-Thompson-sampling} until $\TT_{2,n_{2,T}+1}$--th round. Note that $\TT_{2,n_{2,T}+1}$ is the first time we pull arm $2$ after time $T$. We start by observing the following relation 
\begin{align}
        \label{eq:main-thm-renewal-bdd-2}
        \frac{\gamma_T \log\left(\sum_{j=0}^{n_{2,T}-1}\tau_2(j)\right)}{n_{2,T}}\le\frac{\gamma_T\log(T)}{n_{2,T}}\le\frac{\gamma_T \log\left(\sum_{j=0}^{n_{2,T}}\tau_2(j)\right)}{n_{2,T}}.
\end{align}
It now suffices to show that both the right and left hand side of the above string of inequalities~\eqref{eq:main-thm-renewal-bdd-2} converge to $\frac{\Delta^2}{2}$ in probability, i.e., 
\begin{align}
    \label{eq:main-thm-bdd-limits}
    \frac{\gamma_T \log\left(\sum_{j=0}^{n_{2,T}-1}\tau_2(j)\right)}{n_{2,T}} \inprob \frac{\Delta^2}{2} \qquad \text{and} \qquad 
    \frac{\gamma_T \log\left(\sum_{j=0}^{n_{2,T}}\tau_2(j)\right)}{n_{2,T}} \inprob \frac{\Delta^2}{2}. 
\end{align}
Unfortunately, it is difficult to analyze the waiting time $\tau_2(j)$ directly, instead, we approximate the waiting time $\tau_2(j)$ by a \textit{proxy} $\widetilde{\tau}_2(j)$. To prove~\eqref{eq:main-thm-bdd-limits}, we show that for $n\in\left\{n_{2,T},n_{2,T}+1\right\}$,
\begin{subequations}
    \begin{align}
        \label{eq:conv-tau-tilde}
        \frac{\gamma_T \log\left(\sum_{j=0}^{n-1}\widetilde{\tau}_2(j)\right)}{n_{2,T}} \inprob \frac{\Delta^2}{2} \qquad{\text{and}}
    \end{align}
    \begin{align}
        \label{eq:conv-tau-tilde-&-tau-diff}
        \frac{\gamma_T \log\left(\sum_{j=0}^{n-1}\tau_2(j)\right)}{n_{2,T}}-\frac{\gamma_T \log\left(\sum_{j=0}^{n-1}\widetilde{\tau}_2(j)\right)}{n_{2,T}}\inprob  0  
    \end{align}
\end{subequations}
We do so by proving that under events with \emph{high probability}, $n_{2,t}$ is large when $t$ is large and, as a consequence, we prove $\widetilde{\tau}_2(j)$ is a \emph{good} approximation of $\tau_2(j)$. The complete details of the proof are given in Appendix~\ref{sec:proof-main-thm-app}.
\section{Discussion}
This work takes a step toward bridging the gap between adaptive decision-making and valid statistical inference in bandit settings. We introduced a stabilized variant of Thompson Sampling, called Stable Thompson Sampling, and established its stability and asymptotic normality in the canonical two-armed Gaussian bandit problem. Our theoretical results, supported by empirical evidence, demonstrate that a modest variance inflation enables confidence intervals with nominal coverage while preserving near-optimal regret performance. 

When extended to a general $K$-- armed bandit problem, the Stable Thompson Sampling algorithm retains favorable performance: Proposition~\ref{prop:regret-bd} shows that regret increases only logarithmically relative to standard Thompson Sampling. Our current analysis assumes Gaussian reward distributions. It would be natural to ask whether similar stability and inferential guarantees can be established under more general reward models.

Ultimately, our work underscores that statistical inference and decision-making need not be at odds: with careful algorithmic design, randomized posterior-based methods -- such as Thompson Sampling -- can be stabilized to provide valid confidence guarantees while maintaining strong learning performance. We view this as an initial step toward extending such guarantees to broader classes of adaptive algorithms.

\bibliographystyle{plain}
\bibliography{ref}

\newpage
\appendix
\label{sec:Appendix}

\section{Proof of Main Theorem}
\label{sec:proof-main-thm-app}
In Section~\ref{sec:Delta-positive-main}, we provided a roadmap of the proof of Theorem~\ref{main-thm} for two-armed gaussian bandit with identical arms. We now provide the full proof building on the preceding outline.
\subsection{\texorpdfstring{$\Delta > 0$}{}}
\label{sec:Delta-positive}
Recall, $\TT_{2,k}$ denotes the time of playing arm $2$ for the $k^{\text{th}}$ time with $\TT_{2,1}=1$, $\tau_2(0)=1$ and for $j\ge1$, we defined $\tau_2(j)=\TT_{2,j+1}-\TT_{2,j}$. Let
\begin{align}
    p_2\left(n_{2,t}\; ,t+1\right) &:=\prob\left(A_{t+1}=2\; \Big|\; \F_t\right)\nonumber\\
    &=\prob\left(\widehat{\mu}_{2,t}+\sqrt{\frac{\gamma_T}{n_{2,t}}}Z_{2,t+1}>\widehat{\mu}_{1,t}+\sqrt{\frac{\gamma_T}{n_{1,t}}}Z_{1,t+1}\; \Big|\; \F_t\right)\nonumber\\
    &=\prob\left(V_2(t+1)<B_2\left(n_{2,t},t+1\right)\; \Big|\; \F_t\right)\label{arm-sel-prob}
\end{align}
where $V_2(t+1):=1-\Phi(Z_{2,t+1})$ follows $\text{Unif}(0,1)$ and
\begin{align}
    \label{actual-prob}
    B_2\left(n_{2,t}\; ,t+1\right)&:=1-\Phi\left(\sqrt{n_{2,t}}\left[\frac{\widehat{\Delta}_t}{\sqrt{\gamma_T}}+\frac{Z_{1,t+1}}{\sqrt{n_{1,t}}}\right]\right)
\end{align}
Here, $\widehat{\Delta}_t:=\widehat{\mu}_{1,t}-\widehat{\mu}_{2,t}$. Suppose arm $2$ has been played $j$ times. Consider the random variables $\xi_2(j,i):= V_2(\TT_{2,j}+i)$, $1\le i\le\tau(j)$. By construction, $\left\{\xi_2(j,i)\right\}$ are independent and distributed according to $\text{Unif}(0,1)$. Define
\begin{align}
    \widetilde{p}_2(j)&:=\exp\left\{-j\cdot\frac{\Delta^2}{2\gamma_T}\right\}\label{prob-approx}\\
    \widetilde{\tau}_2(j)&:=\inf\{i\ge1:\xi_2(j,i)<\widetilde{p}_2(j)\}\label{time-approx}
\end{align}
\subsubsection*{Proof of claim~\eqref{eq:conv-tau-tilde}:}
We have
    \begin{align}
        \log\left(\sum_{j=0}^{n-1}\widetilde{\tau}_2(j)\right)&= \log\left(\sum_{j=0}^{n-1}\exp\left\{j\frac{\Delta^2}{2\gamma_T}+\log\left(\widetilde{\tau}_2(j)\widetilde{p}_2(j)\right)\right\}\right)\nonumber \\
        &\le\max_{0\le j\le {n-1}}\left\{j\frac{\Delta^2}{2\gamma_T}+\log\left(\widetilde{\tau}_2(j)\widetilde{p}_2(j)\right)\right\}+\log(n)\nonumber\\
        &\le  \frac{n \Delta^2}{2\gamma_T} + 
        \max_{0\le j\le {n-1}}
        \left\{\log\left(\widetilde{\tau}_2(j)\widetilde{p}_2(j)\right) \right\}
        +\log(n)\label{eq:main-thm-ub}
    \end{align}
    To derive a lower bound on $\log\left(\sum_{j=0}^{n-1} \widetilde{\tau}_2(j)\right)$ we note 
    \begin{align}
        \log\left(\sum_{j=0}^{n-1} \widetilde{\tau}_2(j)\right)&\ge \max_{0\le j\le {n -1}}\left\{j\frac{\Delta^2}{2\gamma_T}+\log\left(\widetilde{\tau}_2(j)\widetilde{p}_2(j)\right)\right\}\nonumber\\
        &\ge \frac{(n - 1) \cdot \Delta^2}{2\gamma_T}+\log\left(\widetilde{\tau}_2({n-1})\widetilde{p}_2({n-1})\right)\label{eq:main-thm-lb}
    \end{align}
    The above two inequalities use the bounds in \eqref{lem:log-exp-bdd}. Additionally, using Lemma~\ref{lem:geom-conv-0-a} and \ref{lem:geom-conv-0-b}, it follows that
    \begin{align}
    \label{eq:log-product-term}
        \frac{ \max \limits_{ 0 \leq j \leq n - 1} \left\{ \log\left(\widetilde{\tau}_2(j)\widetilde{p}_2(j)\right) \right\} }{\sqrt{n}} \inprob 0,\quad \frac{ \log\left(\widetilde{\tau}_2(n-1)\widetilde{p}_2(n-1)\right)}{\sqrt{n}} \inprob 0 \qquad \text{as} \;\; n \rightarrow \infty. 
    \end{align}
    Combining the relations~~\eqref{eq:main-thm-ub}, \eqref{eq:main-thm-lb} and~\eqref{eq:log-product-term} we deduce
    \begin{align}
       \frac{n - 1}{n} \cdot \frac{\gap^2}{2}
       +  \frac{\gamma_T \cdot \log\left(\widetilde{\tau}_2({n-1})\widetilde{p}_2({n-1})\right)}{n}  &\le  \frac{\gamma_T \cdot \log\left(\sum_{j=0}^{n-1} \widetilde{\tau}_2(j)\right) }{n}\nonumber\\
       &\le \frac{\Delta^2}{2} + \frac{ \gamma_T \cdot\max \limits_{ 0 \leq j \leq n - 1} \left\{ \log\left(\widetilde{\tau}_2(j)\widetilde{p}_2(j)\right) \right\} }{n}\label{eqn:combined-ineq}
    \end{align}
    Finally, we show that $n_{2, T}$ is large with probability converging to one. To this end, define the event $\EE^{(T)}_1:=\left\{n_{2,T}\ge\frac{\log(T)}{2\Delta^2}\right\}$. Using the regret bound from Proposition~\ref{prop:regret-bd} with \cite[\textsc{Theorem 2}]{lairobbins}, we deduce that $\lim_{T\to\infty}\prob\left(\EE^{(T)}_1\right)=1$.
    Claim~\eqref{eq:conv-tau-tilde} now follows from~\eqref{eqn:combined-ineq} by noting that $\frac{\gamma_T}{\sqrt{\log T}} \rightarrow 0$ by condition~\eqref{gamma-cond}. 
\subsubsection*{Proof of claim~\eqref{eq:conv-tau-tilde-&-tau-diff}:}
We prove this claim in two steps --- first we approximate  $p_2\left(n_{2,t},t+1\right)$ by $\widetilde{p}_2\left(n_{2,t}\right)$ for large values of $t$. Throughout we assume that the horizon length $T$ is appropriately large.  To do this, let $\varepsilon\in\left(0,\frac{\Delta^2}{2}\right)$ and define
    \begin{align}
        \widetilde{p}_2^{\; +}(j)&:=\exp\left\{-\frac{j}{\gamma_T}\left(\frac{\Delta^2}{2}+\varepsilon\right)\right\}\label{eq:p-tilde-plus}\\
        \widetilde{\tau}_2^{\; +}(j)&:=\inf\{i\ge1:\xi_2(j,i)<\widetilde{p}_2^{\; +}(j)\}\nonumber\\
        \widetilde{p}_2^{\; -}(j)&:=3\exp\left\{-\frac{j}{\gamma_T}\left(\frac{\Delta^2}{2}-\varepsilon\right)\right\}\label{eq:p-tilde-minus}\\
        \widetilde{\tau}_2^{\; -}(j)&:=\inf\{i\ge1:\xi_2(j,i)<\widetilde{p}_2^{\; -}(j)\}\nonumber
    \end{align}
    Additionally, let us consider the following events:
    \begin{align}
        \EE^{(T)}_2&:=\left\{n_{2,t}\ge\frac{\sqrt{\log(T)}}{4\Delta^2}\; \forevery\; \exp\left\{\sqrt{\log(T)}\right\}\le t\le \TT_{2,n_{2,T}+1}\right\}
        \notag \\
        \EE^{(T)}_3&:=\left\{n_{2,t}\le\left(\log(T)\right)^2\; \forevery\; \exp\left\{\sqrt{\log(T)}\right\}\le t\le \TT_{2,n_{2,T}+1}\right\}\notag \\
        \EE^{(T)}_4&:=\left\{\left|\widehat{\mu}_{j,t}-\mu_j\right|\le\sqrt{\frac{3\log\left(\log\left(2n_{j,t}\right)\right)+3\log\left(\log\left(T\right)\right)}{n_{j,t}}}\; \forevery\; t\ge1,\; \forevery\; j=1,2\right\}
        \label{eqn:hp-events}
    \end{align}
    In Lemma \ref{lem:sets-whp} in Appendix~\ref{Proof of Key Lemmas}, we prove that each of these events hold with \textit{high probability}, in the sense that, $\prob\left(\EE_i^{(T)}\right)\to1$ as $T\to\infty$ for $i=2,3,4$.\\
    
\noindent  
  Next, we prove that on the high-probability event $\displaystyle\bigcap_{i=1}^4\EE_i^{(T)}$ the following sandwich relation holds:
    \begin{align}
        \label{eq:prob-approx-bdds}
        \widetilde{p}_2^{\; +}(n_{2,t})\le p_2(n_{2,t},t+1)\le\widetilde{p}_2^{\; -}(n_{2,t}) \qquad  \forevery \exp\left\{\sqrt{\log(T)}\right\} \leq t \leq \TT_{2,n_{2,T}+1}. 
    \end{align}
    
    \noindent Note that, for all $t\ge\exp\left\{\sqrt{\log(T)}\right\}$, the following holds
    \begin{align*}
        n_{1,t}=t-n_{2,t}\ge\exp\left\{\sqrt{\log(T)}\right\}-(\log(T))^2\ge (\log(T))^3
    \end{align*}
    Let $\varepsilon'>0$. First, we prove that $\left|\widehat{\Delta}_t-\Delta\right|$ can be made \textit{small}.
    \begin{align}
        \left|\widehat{\Delta}_t-\Delta\right|&\le\left|\widehat{\mu}_{1,t}-\mu_1\right|+\left|\widehat{\mu}_{2,t}-\mu_2\right|\nonumber\\
        &\le \sqrt{\frac{3\log\left(\log\left(2n_{1,t}\right)\right)+3\log\left(\log\left(T\right)\right)}{n_{1,t}}} +\sqrt{\frac{3\log\left(\log\left(2n_{2,t}\right)\right)+3\log\left(\log\left(T\right)\right)}{n_{2,t}}}\label{eq:delta-bbd}
    \end{align}
    We bound the first term in the following manner
    \begin{align*}
        \frac{3\log\left(\log\left(2n_{1,t}\right)\right) +3\log\left(\log\left(T\right)\right)}{n_{1,t}}&\le\frac{3\log\left(n_{1,t}\right)}{n_{1,t}}+\frac{3\log(\log(T))}{n_{1,t}}\\
        &\le\frac{3\log\left(\left(\log(T)\right)^3\right)}{\left(\log(T)\right)^3}+\frac{3\log\left(\log(T)\right)}{\left(\log(T)\right)^3}
    \end{align*}
    The last inequality uses the result that $\log(x)/x$ is decreasing for $x>e$ and $n_{1,t}>\left(\log(T)\right)^3>e$ for all large $T$. Therefore, for all sufficiently large $T$, $\frac{3\log\left(\left(\log(T)\right)^3\right)}{\left(\log(T)\right)^3}+\frac{3\log\left(\log(T)\right)}{\left(\log(T)\right)^3}\le\left(\frac{\varepsilon'}{4}\right)^2$.\\
    
    \noindent The bound for the second term is rather straightforward
    \begin{align*}
        \frac{3\log\left(\log\left(2n_{2,t}\right)\right)+ 3\log\left(\log\left(T\right)\right)}{n_{2,t}}&\le\frac{3\log\left(\log\left(2T\right)\right)+3\log\left(\log\left(T\right)\right)}{\frac{\sqrt{\log(T)}}{4\Delta_2^2}}
    \end{align*}
    For all large $T$, this is bounded by $<\left(\frac{\varepsilon'}{4}\right)^2$.\\
    Thus, we conclude that for all large $T$, for all $\exp\left\{\sqrt{\log(T)}\right\}\le t\le \TT_{2,n_{2,T}+1}$,
    \begin{align}
        \label{eq:delta-hat-close}
        \Delta-\frac{\varepsilon'}{2}\le \widehat{\Delta}_t\le\Delta+\frac{\varepsilon'}{2}
    \end{align}

    \noindent Now, consider the event $D_{t+1}:=\left\{\sqrt{\frac{\gamma_T}{n_{1,t}}}\cdot\left|Z_{1,t+1}\right|\le\frac{\varepsilon'}{2}\right\}$. On $D_{t+1}\bigcap \F_{t}$, we have
    \begin{align}
        1-\Phi\left(\sqrt{\frac{n_{2,t}}{\gamma_T}}\cdot(\Delta+\varepsilon')\right)&\le1-\Phi\left(\sqrt{\frac{n_{2,t}}{\gamma_T}}\cdot\left[\widehat{\Delta}_t+\sqrt{\frac{\gamma_T}{n_{1,t}}}\cdot Z_{1,t+1}\right]\right)\nonumber\\
        &\le1-\Phi\left(\sqrt{\frac{n_{2,t}}{\gamma_T}}\cdot(\Delta-\varepsilon')\right)\label{eq:p2-bdd-dt}
    \end{align}
    Next, $\prob\left(D_{t+1}^{\text{c}}\right)=2\Phi\left(-\sqrt{\frac{n_{1,t}}{\gamma_T}}\cdot\frac{\varepsilon'}{2}\right)$. For all $t\ge\exp\left\{\sqrt{\log(T)}\right\}$, observe that $\frac{n_{1,t}}{n_{2,t}}\ge\log(T)$. Thus, for all large $T$,
    \begin{align}
        \label{eq:n1-smaller-ord-than-n2}
        \prob\left(D_{t+1}^{\text{c}}\right)\le 2\Phi\left(-\sqrt{\frac{n_{2,t}\log(T)}{\gamma_T}}\cdot\frac{\varepsilon'}{2}\right)\le 2\Phi\left(-\sqrt{\frac{n_{2,t}}{\gamma_T}}\cdot(\Delta-\varepsilon')\right)
    \end{align}
    Recall $B_2(n_{2,t}\; ,t+1) =1-\Phi\left(\sqrt{\frac{n_{2,t}}{\gamma_T}}\cdot\left[\widehat{\Delta}_t+\sqrt{\frac{\gamma_T}{n_{1,t}}}\cdot Z_{1,t+1}\right]\right)$ defined in \eqref{actual-prob}. From \eqref{eq:p2-bdd-dt} and \eqref{eq:n1-smaller-ord-than-n2}, it follows
    \begin{align*}
        1-\Phi\left(\sqrt{\frac{n_{2,t}}{\gamma_T}}\cdot(\Delta+\varepsilon')\right)&\le \prob\left(V_2(t+1)\le B_2(n_{2,t}\; ,t+1),\; D_{t+1}\; \Big|\; \F_t\right) \\
        &\le\prob\left(V_2(t+1)\le B_2(n_{2,t}\; ,t+1)\; \Big|\; \F_t\right)\\
        &=p_2\left(n_{2,t},t+1\right)\\
        &\le\prob\left(V_2(t+1)\le B_2(n_{2,t}\; ,t+1),\; D_{t+1}\; \Big|\; \F_t\right)+\prob\left(D_{t+1}^{\text{c}}\right)\\
        &\le \left[1-\Phi\left(\sqrt{\frac{n_{2,t}}{\gamma_T}}\cdot(\Delta-\varepsilon')\right)\right]+2\Phi\left(-\sqrt{\frac{n_{2,t}}{\gamma_T}}\cdot(\Delta-\varepsilon')\right)\\
        &=3\left[1-\Phi\left(\sqrt{\frac{n_{2,t}}{\gamma_T}}\cdot(\Delta-\varepsilon')\right)\right]
    \end{align*}
    All we are left to prove is $1-\Phi\left(\sqrt{\frac{n_{2,t}}{\gamma_T}}\cdot(\Delta+\varepsilon')\right)\ge \widetilde{p}_2^{\; +}(n_{2,t})$ and $3\left[1-\Phi\left(\sqrt{\frac{n_{2,t}}{\gamma_T}}\cdot(\Delta-\varepsilon')\right)\right]\le \widetilde{p}_2^{\; -}(n_{2,t})$. By condition~\eqref{gamma-cond}, $\frac{n_{2,t}}{\gamma_T}$ can be made arbitrarily large. The proof of \eqref{eq:prob-approx-bdds} is established by taking $\varepsilon'>0$ sufficiently small and applying inequality \eqref{lem:Phi-&-phi-bdd}.\\
    
     \noindent Since $\widetilde{\tau}_2(j)$, $\widetilde{\tau}_2^{\; +}(j)$ and $\widetilde{\tau}_2^{\; -}(j)$ are defined using a common set of random variables $\xi(j,i)$, we have almost surely for all $j$,
    \begin{align}
        \label{eq:order-tau-tilde}
        \widetilde{\tau}_2^{\; -}(j)\le\widetilde{\tau}_2(j)\le\widetilde{\tau}_2^{\; +}(j)
    \end{align}
    Furthermore, \eqref{eq:prob-approx-bdds} implies for all $t\ge\exp\left\{\sqrt{\log(T)}\right\}$
    \begin{align}
        \label{eq:order-tau}
        \widetilde{\tau}_2^{\; -}\left(n_{2,t}\right)\le\tau_2\left(n_{2,t}\right)\le\widetilde{\tau}_2^{\; +}\left(n_{2,t}\right)
    \end{align}
    Let $S_T:=\left\{j\ge0:\tau_2(0)+\ldots+\tau_2(j)\le \exp\left\{\sqrt{\log(T)}\right\}\right\}$ and $S_T^{\text{c}}:=\{0,1,\ldots,n_{2,T}\}\backslash S_T$. $\tau_2(0)=1\le \exp\left\{\sqrt{\log(T)}\right\}$ and $\sum_{j=0}^{n_{2,T}}\tau_2(j)>T>\exp\left\{\sqrt{\log(T)}\right\}$ imply $0\in S_T$ and $n_{2,T}\in S_T^{\text{c}}$ respectively.
    
    \begin{align*}
        \log\left(\sum_{j=0}^{n_{2,T}}\tau_2(j)\right)&= \log\left(\sum_{j\in S_T}\tau_2(j)+\sum_{j\in S_T^{\text{c}}}\tau_2(j)\right)\\
        &\le\log\left(\sum_{j\in S_T}\tau_2(j)\right)+\log\left(\sum_{j\in S_T^{\text{c}}}\tau_2(j)\right)\\
        &\le\sqrt{\log(T)}+\log\left(\sum_{j\in S_T^{\text{c}}}\widetilde{\tau}_2^{\; +}(j)\right)-\log\left(\sum_{j\in S_T^{\text{c}}}\widetilde{\tau}_2^{\; -}(j)\right)+\log\left(\sum_{j=0}^{n_{2,T}}\widetilde{\tau}_2(j)\right)
    \end{align*}
    The last inequality follows from \eqref{eq:order-tau-tilde}, \eqref{eq:order-tau} and the definition of $S_T$. Applying \eqref{lem:log-exp-bdd}, we have
    \begin{align}
        0&\le \log\left(\sum_{j\in S_T^{\text{c}}}\widetilde{\tau}_2^{\; +}(j)\right)-\log\left(\sum_{j\in S_T^{\text{c}}}\widetilde{\tau}_2^{\; -}(j)\right)\nonumber\\
        &\le\max_{j\in S_T^{\text{c}}}\left\{\frac{j}{\gamma_T}\left(\frac{\Delta^2}{2}+\varepsilon\right)+\log\left(\widetilde{\tau}_2^{\; +}(j)\widetilde{p}_2^{\; +}(j)\right)\right\}+\log\left(n_{2,T}\right)\nonumber\\
        &\hspace{4cm}-\max_{j\in S_T^{\text{c}}}\left\{\frac{j}{\gamma_T}\left(\frac{\Delta^2}{2}-\varepsilon\right)+\log\left(3^{-1}\widetilde{\tau}_2^{\; -}(j)\widetilde{p}_2^{\; -}(j)\right)\right\}\nonumber\\
        &\le\frac{n_{2,T}}{\gamma_T}\left(\frac{\Delta^2}{2}+\varepsilon\right)+\max_{j\in S_T^{\text{c}}}\log\left(\widetilde{\tau}_2^{\; +}(j)\widetilde{p}_2^{\; +}(j)\right)+\log\left(n_{2,T}\right)\nonumber\\
        &\hspace{4cm}-\frac{n_{2,T}}{\gamma_T}\left(\frac{\Delta^2}{2}-\varepsilon\right)-\log\left(\widetilde{\tau}_2^{\; +}\left(n_{2,T}\right)\widetilde{p}_2^{\; +}\left(n_{2,T}\right)\right)+\log(3)\nonumber\\
        &\le\frac{n_{2,T}}{\gamma_T}\cdot 2\varepsilon+\max_{j\in S_T^{\text{c}}}\log\left(\widetilde{\tau}_2^{\; +}(j)\widetilde{p}_2^{\; +}(j)\right)+\log\left(n_{2,T}\right)\nonumber\\
        &\hspace{4cm}-\log\left(\widetilde{\tau}_2^{\; +}\left(n_{2,T}\right)\widetilde{p}_2^{\; +}\left(n_{2,T}\right)\right)+\log(3)\label{eq:useful-bdd}
    \end{align}
    Therefore,
    \begin{align}
        &\frac{\gamma_T \log\left(\sum_{j=0}^{n_{2,T}}\tau_2(j)\right)}{n_{2,T}}-\frac{\gamma_T \log\left(\sum_{j=0}^{n_{2,T}}\widetilde{\tau}_2(j)\right)}{n_{2,T}}-\nonumber\\
        \le&\; \frac{\gamma_T\sqrt{\log(T)}}{n_{2,T}}+2\varepsilon+\frac{\gamma_T}{\sqrt{n_{2,T}}}\cdot\frac{\max_{j\in S_T^{\text{c}}}\log\left(\widetilde{\tau}_2^{\; +}(j)\widetilde{p}_2^{\; +}(j)\right)}{\sqrt{n_{2,T}}}+\frac{\gamma_T\log\left(n_{2,T}\right)}{n_{2,T}}\nonumber\\
        &\hspace{4cm}-\frac{\gamma_T}{\sqrt{n_{2,T}}}\cdot\frac{\log\left(\widetilde{\tau}_2^{\; +}\left(n_{2,T}\right)\widetilde{p}_2^{\; +}\left(n_{2,T}\right)\right)}{\sqrt{n_{2,T}}}+\frac{\gamma_T \log(3)}{n_{2,T}}
    \end{align}
    Since $n_{2,T}\left(\ge\frac{\log(T)}{2\Delta^2}\right)$ can be made arbitrarily large on $\displaystyle\bigcap_{i=1}^4\EE_i^{(T)}$, using Lemma \ref{lem:geom-conv-0-a} and \ref{lem:geom-conv-0-b}, it holds that for all large $T$
    \begin{align}
        \label{eq:tau-diff-ub}
        \frac{\gamma_T \log\left(\sum_{j=0}^{n_{2,T}}\tau_2(j)\right)}{n_{2,T}}-\frac{\gamma_T \log\left(\sum_{j=0}^{n_{2,T}}\widetilde{\tau}_2(j)\right)}{n_{2,T}}\le 3\varepsilon
    \end{align}
    We use a similar approach to establish the lower bound:
    \begin{align}
        &\log\left(\sum_{j=0}^{n_{2,T}}\tau_2(j)\right)-\log\left(\sum_{j=0}^{n_{2,T}}\widetilde{\tau}_2(j)\right)\nonumber\\
        \ge\;& \log\left(\sum_{j\in S_T^{\text{c}}}\tau_2(j)\right)-\log\left(\sum_{j\in S_T^{\text{c}}} \widetilde{\tau}_2(j)\right)-\log\left(\sum_{j\in S_T}\widetilde{\tau}_2(j)\right)\nonumber\\
        \ge\;&-\left[\log\left(\sum_{j\in S_T^{\text{c}}}\widetilde{\tau}_2^{\; +}(j)\right)-\log\left(\sum_{j\in S_T^{\text{c}}}\widetilde{\tau}_2^{\; -}(j)\right)\right]-\log\left(\sum_{j\in S_T}\widetilde{\tau}_2(j)\right)\label{eq:tau-diff-lb-addn-term}
    \end{align}
    To bound first term, we use \eqref{eq:useful-bdd}.\\
    Let us define $T^\ast:=\left\lfloor\exp\left\{\sqrt{\log(T)}\right\}\right\rfloor$. Since $\sum_{j=0}^{n_{2,t}} \tau_2(j)>t$ holds for any $t$, we conclude that $\max S_T< n_{2,T^\star}\le n_{2,T}$. Therefore, applying \eqref{lem:log-exp-bdd} on the second term in \eqref{eq:tau-diff-lb-addn-term}, we have
    \begin{align*}
        \frac{\gamma_T \log\left(\sum_{j\in S_T}\widetilde{\tau}_2(j)\right)}{n_{2,T}}&\le \frac{\displaystyle\gamma_T\left[ \max_{j\in S_T}\left\{\frac{j}{\gamma_T}\left(\frac{\Delta^2}{2}+\varepsilon\right)+\log\left(\widetilde{\tau}_2(j)\widetilde{p}_2(j)\right)\right\}+\log\left(n_{2,T}\right)\right]}{n_{2,T}}\\
        &\le \frac{n_{2,T^\star}}{n_{2,T}}\left(\frac{\Delta^2}{2}+\varepsilon\right)+\frac{\gamma_T}{\sqrt{n_{2,T}}}\cdot\frac{\max_{j\in S_T}\log\left(\widetilde{\tau}_2(j)\widetilde{p}_2(j)\right)}{\sqrt{n_{2,T}}}+\frac{\gamma_T\log\left(n_{2,T}\right)}{n_{2,T}}
    \end{align*}
    From \eqref{eq:regret-bd-sts}, we have
    \begin{align*}
        \E\left[n_{2,T^\star}\right]\le \frac{\kappa_1\log\left(\frac{T^\star\Delta^2}{\gamma_T}+\kappa_2\right)\gamma_T}{\Delta^2}+\frac{\kappa_3\gamma_T}{\Delta^2}
    \end{align*}
    The above bound indicates $\frac{\E\left[n_{2,T^\star}\right]}{\log(T)}$ converges to $0$ as $T\to\infty$.\\
    Recall the event $\EE^{(T)}_1=\left\{n_{2,T}\ge\frac{\log(T)}{2\Delta^2}\right\}$ introduced in the proof of claim \eqref{eq:conv-tau-tilde}. Since $\displaystyle\lim_{T\to\infty}\prob\left(\EE^{(T)}_1\right)=1$, this implies $\frac{n_{2,T^\star}}{n_{2,T}}\inprob0$. So, on $\displaystyle\bigcap_{i=1}^4\EE_i^{(T)}$,
    \begin{align}
        \label{eq:tau-diff-lb}
        \frac{\gamma_T \log\left(\sum_{j=0}^{n_{2,T}}\tau_2(j)\right)}{n_{2,T}}-\frac{\gamma_T \log\left(\sum_{j=0}^{n_{2,T}}\widetilde{\tau}_2(j)\right)}{n_{2,T}}\ge -3\varepsilon
    \end{align}
    for large $T$. Thus, for arbitrary $\varepsilon>0$
    \begin{align*}
        &\limsup_{T\to\infty}\prob\left(\left|\frac{\gamma_T \log\left(\sum_{j=0}^{n_{2,T}}\tau_2(j)\right)}{n_{2,T}}-\frac{\gamma_T \log\left(\sum_{j=0}^{n_{2,T}}\widetilde{\tau}_2(j)\right)}{n_{2,T}}\right|>3\varepsilon\right)\\
        &\le\;\lim_{T\to\infty}\sum_{i=1}^4\left[ 1-\prob\left(\EE_i^{(T)}\right)\right]=0
    \end{align*}
    The other case $\frac{\gamma_T \log\left(\sum_{j=0}^{n_{2,T}-1}\tau_2(j)\right)}{n_{2,T}}-\frac{\gamma_T \log\left(\sum_{j=0}^{n_{2,T}-1}\widetilde{\tau}_2(j)\right)}{n_{2,T}}\inprob0$ follows similarly. This completes the proof for two-armed scenario with unique optimal arm.
\subsection{\texorpdfstring{$2$ optimal arms}{}}
\noindent We now prove the general case of Theorem~\ref{main-thm}. Assume $K\ge3$ and $\mu_1=\mu_2>\mu_a$ for all $a>2$, that is, arms $1$ and $2$ are optimal (the proof for unique optimal arm is similar and we skip it). The following lemma proves stability of the optimal arms in this setting.
\begin{lemma}
    \label{lem:opt-arm-lim-exten}
    For all $t\ge T^\star$, we have
    \begin{align*}
        \frac{n_{1,t}}{t}\inprob \frac{1}{2}\qquad\text{ and }\qquad\frac{n_{2,t}}{t}\inprob \frac{1}{2}
    \end{align*}
    as $T\to\infty$.
\end{lemma}
\begin{proof}
    First observe that for $j\in[K]$,
    \begin{align}
        \label{eq:mart-diff-0}
        \frac{n_{j,t}}{t}-\frac{\sum_{i=1}^t p_{j,i}^{(T)}}{t}\inprob0\qquad\text{as}\quad T\to\infty
    \end{align}
    where $p_{j,i}^{(T)}:=\prob\left(A_i=j\Big|\F_{i-1}\right)$.\\
    Recall $\EE^{(T)}:=\left\{\left|\widehat{\mu}_{j,t}-\mu_j\right|\le\sqrt{\frac{7\log(\log(T))}{n_{j,t}}}\; \forevery\; t\in[T],\; \forevery\; j=1,2\right\}$. On this event,
    \begin{align}
        p_{1,i}^{(T)}&=\prob\left(A_i=1\Big|\F_{i-1}\right)\nonumber\\
        &=\prob\left(\widehat{\mu}_{1,i-1}+\sqrt{\frac{\gamma_{T}}{n_{1,i-1}}}Z_{1,i}\ge \widehat{\mu}_{j,i-1}+\sqrt{\frac{\gamma_{T}}{n_{j,i-1}}}Z_{j,i}\; \forall\; j\in[K] \Big|\F_{i-1}\right)\nonumber\\
        &\le\prob\left(\widehat{\mu}_{1,i-1}+\sqrt{\frac{\gamma_{T}}{n_{1,i-1}}}Z_{1,i}\ge \widehat{\mu}_{2,i-1}+\sqrt{\frac{\gamma_{T}}{n_{2,i-1}}}Z_{2,i} \Big|\F_{i-1}\right)\nonumber\\
        &\le\Phi\left(\sqrt{\frac{14\log(\log(T))}{\gamma_T}}\right)\label{eq:upper-bd-cond-prob-exten}
    \end{align}
    The convergence result in \eqref{eq:mart-diff-0} and inequality~\eqref{eq:upper-bd-cond-prob-exten} follow from the argument in Section~\ref{sec:Delta-zero}.\\
    Now, observe that $\Phi\left(\sqrt{\frac{14\log(\log(T))}{\gamma_T}}\right)\downarrow1/2$ as $T\to\infty$. This implies, on $\EE^{(T)}$,
    \begin{align*}
        \frac{\sum_{i=1}^t p_{1,i}^{(T)}}{t}\le\frac{1}{2}+\alpha_T
    \end{align*}
    where $\alpha_T\downarrow0$ as $T\to\infty$. By similar arguments, we conclude, on $\EE^{(T)}$,
    \begin{align*}
        \frac{\sum_{i=1}^t p_{2,i}^{(T)}}{t}\le\frac{1}{2}+\alpha_T
    \end{align*}
    From the regret bound in Proposition~\ref{prop:regret-bd}, it follows $\E\left[\frac{n_{a,T}}{\log(T)^2}\right]\to0$ for all $a>2$ as $T\to\infty$. Since $t\ge T^\star\ge (\log T)^2$, using \eqref{eq:mart-diff-0} and \textit{Slutsky's Theorem}, we conclude for all $a>2$,
    \begin{align*}
        \E\left[\frac{\sum_{i=1}^t p_{a,i}^{(T)}}{t}\right]\to0
    \end{align*}
    as $T\to\infty$. By Lemma~\ref{lem:bd-implies-conv}, it follows, for $j=1,2$, we have $\frac{\sum_{i=1}^t p_{j,i}^{(T)}}{t}\inprob\frac{1}{2}$. Using \eqref{eq:mart-diff-0}, we reach the desired conclusion.
\end{proof}
\noindent Now, we show stability for suboptimal arms, WLOG, for arm $3$, that is, we will prove
\begin{align*}
    \frac{n_{3,T}}{\gamma_T\log(T)}\inprob\frac{2}{\Delta_3^2}
\end{align*}
To this end, let us define the following events:
\begin{align}
    \GG^{(T)}_1&:=\left\{n_{1,t},n_{2,t}\ge\frac{\exp\left\{\sqrt{\log(T)}\right\}}{3}\; \forevery \exp\left\{\sqrt{\log(T)}\right\}\le t\le \TT_{3,n_{3,T}+1}\right\}\notag\\
    \GG^{(T)}_2&:=\left\{n_{j,t}\ge\frac{\sqrt{\log(T)}}{4\Delta_j^2}\; \forevery \exp\left\{\sqrt{\log(T)}\right\}\le t\le \TT_{3,n_{3,T}+1},\; \forevery j\ge3\right\}\notag \\
    \GG^{(T)}_3&:=\left\{n_{j,t}\le\left(\log(T)\right)^2\; \forevery \exp\left\{\sqrt{\log(T)}\right\}\le t\le \TT_{3,n_{3,T}+1},\; \forevery j\ge3\right\}\notag \\
    \GG^{(T)}_4&:=\left\{\left|\widehat{\mu}_{j,t}-\mu_j\right|\le\sqrt{\frac{3\log\left(\log\left(2n_{j,t}\right)\right)+3\log\left(\log\left(T\right)\right)}{n_{j,t}}}\; \forevery t\ge1,\; \forevery j\in[K]\right\}
    \label{eqn:hp-events-extension}
\end{align}
An immediate consequence of Lemma~\ref{lem:opt-arm-lim-exten} is that $\prob\left(\GG^{(T)}_1\right)\to1$ as $T\to\infty$. For the rest of the events, $\prob\left(\GG^{(T)}_u\right)\to1$ as $T\to\infty$ can be proved by imitating the proof of Lemma~\ref{lem:sets-whp} in Section~\ref{Proof of Key Lemmas}, $u=2,3,4$.\\

\noindent Fix $\varepsilon\in\left(0,\frac{\Delta_3^2}{2}\right)$ and define
\begin{align*}
    p_3(n_{3,t},t+1)&=\prob\left(A_{t+1}=3\;\Big|\;\F_{t}\right)\\
    \widetilde{p}_2^{\; +}(n_{3,t})&=\exp\left\{-\frac{n_{3,t}}{\gamma_T}\left(\frac{\Delta^2}{2}+\varepsilon\right)\right\}\\
    \widetilde{p}_2^{\; -}(n_{3,t})&=3\exp\left\{-\frac{n_{3,t}}{\gamma_T}\left(\frac{\Delta^2}{2}-\varepsilon\right)\right\}
\end{align*}
Observe that it is sufficient to show that with probability tending to $1$,
\begin{align}
    \label{eq:prob-approx-bdds-exten}
    \widetilde{p}_3^{\; +}(n_{3,t})\le p_3(n_{3,t},t+1)\le\widetilde{p}_3^{\; -}(n_{3,t}) \qquad  \forevery \exp\left\{\sqrt{\log(T)}\right\} \leq t \leq \TT_{3,n_{3,T}+1}. 
\end{align}
This is because the rest of the proof can be established using a similar line of argument used in Section~\ref{sec:Delta-positive} (from equation~\eqref{eq:order-tau-tilde} onwards).\\
Let $\displaystyle\Delta^\star:=\min_{j\ge3}\Delta_j$ and $\displaystyle\Delta_0:=\max_{j\ge3}\Delta_j$. Select arbitrary $\displaystyle\delta\in\left(0,\Delta^\star\right)$. For $M_T:=\left[\frac{\sqrt{\log(T)}}{4\gamma_T\Delta_0^2}\right]^{\frac{1}{2}}$, let us consider the event
\begin{align}
    \label{eq:good-event-exten}
    \LL_{t+1}:=\left\{\left|\frac{Z_j(t+1)}{M_T}\right|\le\frac{\delta}{4}\;\forevery j\neq3\right\}
\end{align}
By condition~\eqref{gamma-cond}, it follows $M_T\to\infty$ and thus $\prob(\LL_{t})\ge\frac{1}{2}$ for all large $T$. Now, let $\theta_{j,t+1}:=\widehat{\mu}_{j,t}+\sqrt{\frac{\gamma_T}{n_{j,t}}}Z_{j}(t+1)$ denote the posterior sample corresponding to arm $j$ after $t$ rounds.
\begin{align*}
    p_3(n_{3,t},t+1)&=\prob\left(A_{t+1}=3\;\Big|\;\F_{t}\right)\\
    &=\prob\left(\theta_{3,t+1} >\theta_{j,t+1}\;\forevery j\neq3\;\Big|\;\F_{t}\right)\\
    &\le\prob\left(\theta_{3,t+1} >\theta_{1,t+1}\;\Big|\;\F_{t}\right)\\
    &\le \widetilde{p}_3^{\; -}(n_{3,t})
\end{align*}
where the last inequality is immediate from the two-armed scenario (see Section~\ref{sec:Delta-positive}).\\

\noindent Now, we focus on the lower bound. Observe that on the event $\bigcap_{u=1}^4 \GG^{(T)}_u$ and $\LL_{t+1}$, we have for all $j\ge3$,
\begin{align*}
    \widehat{\mu}_{j,t}+\sqrt{\frac{\gamma_T}{n_{j,t}}}Z_j(t+1)&\le \mu_j+\sqrt{\frac{3\log\log(2n_{j,t})+3\log\log(T)}{n_{j,t}}}+\sqrt{\frac{\gamma_T}{n_{j,t}}}Z_j(t+1)\\
    &\le \mu_j+\sqrt{\frac{3\log\log\left(\frac{\sqrt{\log(T)}}{2\Delta_j^2}\right)+3\log\log(T)}{\frac{\sqrt{\log(T)}}{4\Delta_j^2}}}+\sqrt{\frac{\gamma_T}{\frac{\sqrt{\log(T)}}{4\Delta_j^2}}}Z_j(t+1)\\
    &\le\mu_j+\frac{\delta}{4}+\frac{Z_j(t+1)}{M_T}\\
    &\le\mu_j+\frac{\delta}{2}\\
    &\le\mu_1-\frac{\delta}{2}\\
    &\le\mu_1-\sqrt{\frac{3\log\log(2n_{1,t})+3\log\log(T)}{n_{1,t}}}+\sqrt{\frac{\gamma_T}{n_{1,t}}}Z_1(t+1)\\
    &\le\widehat{\mu}_{1,t}+\sqrt{\frac{\gamma_T}{n_{1,t}}}Z_1(t+1)
\end{align*}
for all large $T$. Thus, it follows, for all large $T$, on the event $\bigcap_{u=1}^4 \GG^{(T)}_u$, we have
\begin{align*}
    &\prob\left(\theta_{3,t+1}> \theta_{j,t+1}\forevery\;j\neq3,\;\LL_{t+1}\Big|\;\F_{t}\right)\\
    \ge\;&\prob\left(\theta_{3,t+1} >\theta_{1,t+1},\;\theta_{2,t+1},\;\LL_{t+1}\Big|\;\F_{t}\right)\\
    \ge\;&\prob\left(\theta_{3,t+1} >\widehat{\mu}_{1,t}+\frac{\delta}{4},\;\widehat{\mu}_{2,t}+\frac{\delta}{4},\; \LL_{t+1} \Big|\;\F_{t}\right)\\
    \overset{(a)}{=}\;&\prob\left(\theta_{3,t+1} >\widehat{\mu}_{1,t}+\frac{\delta}{4},\;\widehat{\mu}_{2,t}+\frac{\delta}{4}\; \Big|\;\F_{t}\right)\times\prob\left(\LL_{t+1}\right)\\
    \ge\;&\prob\left(\theta_{3,t+1} >\widehat{\mu}_{1,t}+\frac{\delta}{4},\;\widehat{\mu}_{2,t}+\frac{\delta}{4}\; \Big|\;\F_{t}\right)\times\frac{1}{2}\\
    \ge\;&\frac{1}{2}\times\prob\left(\sqrt{\frac{\gamma_T}{n_{3,t}}}Z_3(t+1)\ge\widehat{\mu}_{1,t}-\widehat{\mu}_{3,t}+\frac{\delta}{4},\widehat{\mu}_{2,t}-\widehat{\mu}_{3,t}+\frac{\delta}{4}\; \Big|\;\F_{t}\right)
\end{align*}
The equality in $(a)$ holds because the distribution of $\theta_{3,t+1}$ given $\F_t$ is independent of $\LL_{t+1}$.\\
Now, note that for $j=1,2$, on $\GG_1^{(T)}\bigcap \GG_2^{(T)} \bigcap \GG_4^{(T)}$, it holds
\begin{align*}
    \left|\widehat{\mu}_{j,t}-\widehat{\mu}_{3,t}\right|&\le \mu_1-\mu_3+\sqrt{\frac{3\log\left(\log\left(2n_{j,t}\right)\right)+3\log\left(\log\left(T\right)\right)}{n_{j,t}}}\\
    &\hspace{5cm} +\sqrt{\frac{3\log\left(\log\left(2n_{3,t}\right)\right)+3\log\left(\log\left(T\right)\right)}{n_{3,t}}}
\end{align*}
and $n_{j,t}\ge\frac{\exp\left\{\sqrt{\log(T)}\right\}}{3}$ and $n_{3,t}\ge\frac{\sqrt{\log(T)}}{4\Delta_3^2}$.\\

\noindent Thus, for all large $T$
\begin{align*}
    &\frac{1}{2}\times\prob\left(\sqrt{\frac{\gamma_T}{n_{3,t}}}Z_3(t+1)\ge\widehat{\mu}_{1,t}-\widehat{\mu}_{3,t}+\frac{\delta}{4},\widehat{\mu}_{2,t}-\widehat{\mu}_{3,t}+\frac{\delta}{4}\; \Big|\;\F_{t}\right)\\
    \ge\;&\frac{1}{2}\times\prob\left(\sqrt{\frac{\gamma_T}{n_{3,t}}}Z_3(t+1)\ge\mu_1-\mu_3+\frac{\delta}{2}\; \Big|\;\F_{t}\right)\\
    =\;&\frac{1}{2}\left[1-\Phi\left(\sqrt{\frac{n_{3,t}}{\gamma_T}}\left(\Delta_3+\frac{\delta}{2}\right)\right)\right]\\
    \ge\;&\widetilde{p}_2^{\; +}(n_{3,t})
\end{align*}
by choosing sufficiently small $\delta$. This completes the proof.
\section{Proof of Key Lemmas}
\label{Proof of Key Lemmas}
In this Appendix we state and prove Lemma~\ref{lem:sets-whp} -- the key lemma used in the main section of the paper.
\begin{lemma}
    \label{lem:sets-whp}
    The events $\left\{ \EE_i^{(T)} \right\}_{i = 2}^4$ defined in~\eqref{eqn:hp-events} satisfies $\displaystyle \lim_{T\to\infty}\prob\left(\EE_i^{(T)}\right)=1$.
\end{lemma}
\begin{proof}
     Recall $T^\star=\left\lfloor\exp\left\{\sqrt{\log(T)}\right\}\right\rfloor$ defined in the proof of claim~\eqref{eq:conv-tau-tilde-&-tau-diff}. From \eqref{eq:regret-bd-sts}, we have $\E\left[n_{2,T^\star}\right]\le \frac{\kappa_1\log\left(\frac{T^\star\Delta^2}{\gamma_T}+\kappa_2\right)\gamma_T}{\Delta^2}+\frac{\kappa_3\gamma_T}{\Delta^2}$. Clearly, $\E\left[n_{2,T^\star}\right]=o\left(\left(T^\star\right)^a\right)$  for all $a>0$. Therefore, \cite[\textsc{Theorem $2$}]{lairobbins} implies $\prob\left(n_{2,T^\star}\ge\frac{\log(T^\star)}{2\Delta^2}\right)$ converges to $1$ as $T\to\infty$. Now, on the event $\left\{n_{2,T^\star}\ge\frac{\log(T^\star)}{2\Delta^2}\right\}$, for all $t\ge \exp\left\{\sqrt{\log(T)}\right\}$, it holds that
    \begin{align*}
        n_{2,t}\ge n_{2,T^\star}\ge\frac{\log(T^\star)}{2\Delta^2}=\frac{\log\left(\left\lfloor\exp\left\{\sqrt{\log(T)}\right\}\right\rfloor\right)}{2\Delta^2}\ge\frac{\log\left(\exp\left\{\sqrt{\log(T)}\right\}-1\right)}{2\Delta^2}\ge\frac{\sqrt{\log(T)}}{4\Delta^2}
    \end{align*}
    for all large $T$. So, we conclude that $\displaystyle\lim_{T\to\infty}\prob\left(\EE_2^{(T)}\right)=1$.\\

     From \eqref{eq:regret-bd-sts}, we have $\E\left[n_{2,T}\right]\le \frac{\kappa_1\log\left(\frac{T\Delta^2}{\gamma_T}+\kappa_2\right)\gamma_T}{\Delta^2}+\frac{\kappa_3\gamma_T}{\Delta^2}$. Note that $\lim\E\left[\frac{n_{2,T}}{(\log(T))^2}\right]=0$. Hence, $\frac{n_{2,T}}{(\log(T))^2}\inprob0$. This implies, $\prob\left(n_{2,T}\le\frac{(\log(T))^2}{2}\right)$ converges to $1$ as $T\to\infty$. Since, $n_{2,t}\le n_{2,T}+1$ for all $1\le t\le \TT_{2,n_{2,T}+1}$, therefore, on the event $\left\{n_{2,T}\le\frac{(\log(T))^2}{2}\right\}$, we have $n_{2,t}\le n_{2,T}+1\le \frac{(\log(T))^2}{2}+1\le (\log(T))^2$ for all large $T$. Hence, $\displaystyle\lim_{T\to\infty}\prob\left(\EE_3^{(T)}\right)=1$.\\

     \noindent
     From \cite[\textsc{Equation 1.2}]{howard2020timeuniformchernoffboundsnonnegative}, it follows that for $j=1,2$, for all $\delta_1>0$,
    \begin{align*}
        \prob\left(\left\{|\mu_{j,t}-\mu_j|\le\sqrt{\frac{3\log\left(\log\left(2n_{j,t}\right)\right)+3\log\left(1/\delta_1\right)}{n_{j,t}}}\; \forevery\; t\ge1\right\}\right)\ge 1-\delta_1
    \end{align*}
    Setting $\delta_1=\frac{1}{\log(T)}$ in the upper bound, we have
    \begin{align*}
        &\prob\left(\EE_4^{(T)}\right)\\
        =\;&\prob\left(\left\{\left|\widehat{\mu}_{j,t}-\mu_j\right|\le\sqrt{\frac{3\log\left(\log\left(2n_{j,t}\right)\right)+3\log\left(\log\left(T\right)\right)}{n_{j,t}}}\; \forevery\; t\ge1,\; \forevery\; j=1,2\right\}\right)\\
        \ge\;& \left(1-\frac{1}{\log(T)}\right)^2\\
        \ge\;& 1-\frac{2}{\log(T)}
    \end{align*}
    This completes the proof.
\end{proof}

\section{Proof of Proposition~\ref{prop:regret-bd}}
\label{sec:proof-of-regret-bd}
In this section we state and prove a regret bound for Algorithm~\ref{alg:Modified-Thompson-sampling}. The proof of the regret bound draws inspiration from the work of \cite{agrawal2017near}, although some calculations need to be modified. We provide a detailed proof of the regret bound here for completeness. 

\subsection*{Proof overview:}
    Observe that the arm selection rule in Algorithm~\ref{alg:Modified-Thompson-sampling} can be rephrased as follows
    \begin{subequations}
        \begin{align}
            A_t&=\argmax_{j\in[K]} \left\{\widehat{\mu}_{j,t-1}+\sqrt{\frac{\gamma_T}{n_{j,t-1}}}Z_{j,t}\right\}\label{eq:equivalent-formulation-a}\\
            &=\argmax_{j\in[K]} \left\{\frac{\widehat{\mu}_{j,t-1}}{\sqrt{\gamma_T}}+\frac{Z_{j,t}}{\sqrt{n_{j,t-1}}}\right\}\label{eq:equivalent-formulation-b}
        \end{align}
    \end{subequations}
    The arm selection rule in \eqref{eq:equivalent-formulation-b} is equivalent to running the original Thompson Sampling procedure on $K$-- armed Gaussian bandit problem (Algorithm~\ref{algo:TS}) where the reward for pull of the $j^{\text{th}}$ arm is Gaussian with mean $\frac{\mu_j}{\sqrt{\gamma_T}}$ and variance $\frac{1}{\gamma_T}$, and the posterior distribution of the mean is Gaussian with mean $\frac{\widehat{\mu}_{j,t-1}}{\sqrt{\gamma_T}}$ and variance $\frac{1}{n_{j,t-1}}$. Note that the posterior variance does not scale with the reward variance -- this is the key difference between original Thompson Sampling and Stable Thompson Sampling.
    
    \noindent We now consider the equivalent formulation as stated in \eqref{eq:equivalent-formulation-b}: assume the reward distributions for the arms are $\NN\left(\mu_j^{(T)},\sigma_T^2\right)$, $j\in[K]$ where $\sigma_T\le1/2$ and the posterior distribution of the mean for the $j^{\text{th}}$ arm at time $t$ is $\NN\left(\widehat{\mu}_{j,t-1},1/n_{j,t-1}\right)$. We prove the following lemma to upper bound the expected number of pulls at time $\horizon$ for a sub-optimal arm $j$ with arm mean-gap $\displaystyle\Delta_{j,T}\left(=\max_{a\in[K]}\mu_a^{(T)}-\mu_j^{(T)}\right)$.
    \begin{lemma}
    \label{lem:regret-bd-shipra-goyal}
    For $\sigma_T\le1/2$, the expected number of pulls at time $\horizon$ of any sub-optimal arm $j$ with arm mean gap $\Delta_{j,T}>0$ satisfies
    \begin{align}
        \E[n_{j,\horizon}]\le\left(8\times 10^4\right)\cdot\frac{\log(\horizon\Delta_{j,T}^2+300)}{\Delta_{j,T}^2}+\frac{24}{\Delta_{j,T}^2}
    \end{align}
    \end{lemma}
    \noindent The details of the proof is presented in Section~\ref{sec:regret-bd-proof-details}.\\
    
    \noindent Plugging in $\Delta_{j,T}=\frac{\Delta_j}{\sqrt{\gamma_T}}$ and $\sigma_T=\frac{1}{\sqrt{\gamma_T}}$ (recall $\gamma_T\to\infty$ as $T\to\infty$) for Stable Thompson Sampling, it follows
    \begin{align}
        \label{eq:regret-bd-sts}
        \E[n_{j,\horizon}] \le\frac{\kappa_1\log\left(\frac{\horizon\Delta_j^2}{\gamma_T}+\kappa_2\right)\gamma_T}{\Delta_j^2}+\frac{\kappa_3\gamma_T}{\Delta_j^2}
    \end{align}
    The regret bound is immediate from the observation that $\displaystyle\regret(T)=\sum_{j:\Delta_j>0}\E[n_{j,T}]\times\Delta_j$.

\subsection{Proof of Lemma~\ref{lem:regret-bd-shipra-goyal}}
\label{sec:regret-bd-proof-details}
Assume arm $1$ is optimal. Let us consider any other sub-optimal arm, say WLOG, arm $2$.\\
Let $\theta_j^{(t)}$ be the posterior sample for arm $j$ at time $t$, that is, $\theta_j^{(t)}\sim\NN\left(\widehat{\mu}_{j,t-1},\frac{1}{n_{j,t-1}}\right)$. We define the following two events: $E_\theta^{(t)}:=\left\{\theta_2^{(t)}\le y_2\right\}$ and $E_{\mu}^{(t)}:=\left\{\widehat{\mu}_{2,t}\le x_2\right\}$, where $y_2=\mu_1^{(T)}-\frac{\Delta_{2,T}}{3}$ and $x_2=\mu_2^{(T)}+\frac{\Delta_{2,T}}{3}$.\\

\noindent We decompose $\E[n_{2,\horizon}]$ as
\begin{align*}
    \E[n_{2,\horizon}]&= \sum_{t=1}^\horizon\prob(A_t=2)\nonumber\\
    &=\underbrace{\sum_{t=1}^\horizon \prob\left(A_t=2,E_\mu^{(t)},E_\theta^{(t)}\right)}_{\mytag{(E1)}{termA}}+\underbrace{\sum_{t=1}^\horizon \prob\left(A_t=2,E_\mu^{(t)},\overline{E_\theta^{(t)}}\right)}_{\mytag{(E2)}{termB}}+\underbrace{\sum_{t=1}^\horizon \prob\left(A_t=2,\overline{E_\mu^{(t)}}\right)}_{\mytag{(E3)}{termC}}
\end{align*}
\subsubsection{Bound for \ref{termA}}
Let $q_t:=\prob\left(\theta_1^{(t)}>y_2\mid\F_{t-1}\right)$. This is a random variable determined by $\F_{t-1}$.
\begin{lemma}
    \label{lem:regret-good-case-bdd}
    For all instantiations $F_{t-1}$ of $\F_{t-1}$,
    \begin{align*}
        \prob\left(A_t=2,E_\theta^{(t)},E_\mu^{(t)}\; \Big|\; F_{t-1}\right)\le\frac{1-q_t}{q_t}\cdot\prob\left(A_t=1,E_\theta^{(t)},E_\mu^{(t)}\; \Big|\; F_{t-1}\right)
    \end{align*}
\end{lemma}
\begin{proof}
    First, observe that whether $\E_\mu^{(t)}$ is true or not is determined by the instantiation $F_{t-1}$ of $\F_{t-1}$. If $F_{t-1}$ is such that $\E_\mu^{(t)}$ doesn't hold, then the left hand side in Lemma \ref{lem:regret-good-case-bdd} is zero and the result is trivially true. So, it is sufficient to prove the result for all such history $F_{t-1}$ for which $\E_\mu^{(t)}$ is true. Given $E_\theta^{(t)}$, $\{A_t=2\}$ implies $\{\theta_j^{(t)}\le y_2\}$ for all $j\in[K]$. Therefore
    \begin{align*}
        \prob\left(A_t=2\; \Big|\; E_\theta^{(t)},\F_{t-1}=F_{t-1}\right)&\le\prob\left(\theta_j^{(t)}\le y_2\; \forevery\; j\in[K]\; \Big|\; E_\theta^{(t)},\F_{t-1}=F_{t-1}\right)\\
        &=\prob\left(\theta_1^{(t)}\le y_2\; \Big|\; E_\theta^{(t)},\F_{t-1}=F_{t-1}\right)\\
        &\hspace{1cm}\times\prob\left(\theta_j^{(t)}\le y_2\; \forevery\; j\in[K]\backslash\{1\}\; \Big|\; E_\theta^{(t)},\F_{t-1}=F_{t-1}\right)\\
        &=(1-q_t)\cdot\prob\left(\theta_j^{(t)}\le y_2\; \forevery\; j\in[K]\backslash\{1\}\; \Big|\; E_\theta^{(t)},\F_{t-1}=F_{t-1}\right)
    \end{align*}
    On the other hand
    \begin{align*}
        \prob\left(A_t=1\; \Big|\; E_\theta^{(t)},\F_{t-1}=F_{t-1}\right)&\ge\prob\left(\theta_1^{(t)}>y_2\ge\theta_j^{(t)}\; \forevery\; j\in[K]\backslash\{1\}\; \Big|\; E_\theta^{(t)},\F_{t-1}=F_{t-1}\right)\\
        &=\prob\left(\theta_1^{(t)}>y_2|E_\theta^{(t)},\F_{t-1}=F_{t-1}\right)\\
        &\hspace{1cm}\times\prob\left(\theta_j^{(t)}\le y_2\; \forevery\; j\in[K]\backslash\{1\}\; \Big|\; E_\theta^{(t)},\F_{t-1}=F_{t-1}\right)\\
        &=q_t\cdot\prob\left(\theta_j^{(t)}\le y_2\; \forevery\; j\in[K]\backslash\{1\}\; \Big|\; E_\theta^{(t)},\F_{t-1}=F_{t-1}\right)
    \end{align*}
    Thus the lemma holds.
\end{proof}
\noindent From Lemma \ref{lem:regret-good-case-bdd}, we have
\begin{align*}
    \sum_{t=1}^\horizon \prob\left(A_t=2,E_\mu^{(t)},E_\theta^{(t)}\right)&=\sum_{t=1}^\horizon\E\left[\prob\left(A_t=2,E_\mu^{(t)},E_\theta^{(t)}\; \Big|\; \F_{t-1}\right)\right]\\
    &\le\sum_{t=1}^\horizon \E\left[\frac{1-q_t}{q_t}\cdot\prob\left(A_t=1,E_\mu^{(t)},E_\theta^{(t)}\; \Big|\; \F_{t-1}\right)\right]\\
    &=\sum_{t=1}^\horizon \E\left[\frac{1-q_t}{q_t}\cdot\I\left\{A_t=1,E_\mu^{(t)},E_\theta^{(t)}\right\}\right]
\end{align*}
The last equality follows from the fact that $q_t$ is $\F_{t-1}$--measurable, and tower property.

Recall that $\TT_{j,k}$ denote the time of playing the $j^{\text{th}}$ arm for the $k^{\text{th}}$ time. Observe that $q_t$ changes only when the distribution of $\theta_1^{(t)}$ changes. Also, since $\TT_{1,\horizon}\ge \horizon$, we have
\begin{align*}
    \sum_{t=1}^\horizon \E\left[\frac{1-q_t}{q_t}\cdot\I\left\{A_t=1,E_\mu^{(t)},E_\theta^{(t)}\right\}\right]&=\sum_{k=0}^{\horizon-1}\E\left[\frac{1-q_{\TT_{1,k}+1}}{q_{\TT_{1,k}+1}}\sum_{t=\TT_{1,k}+1}^{\TT_{1,k+1}}\I\left\{A_t=1,E_\mu^{(t)},E_\theta^{(t)}\right\}\right]\\
    &\le \sum_{k=0}^{\horizon-1}\E\left[\frac{1-q_{\TT_{1,k}+1}}{q_{\TT_{1,k}+1}}\right]\\
    &=\sum_{k=0}^{\horizon-1}\E\left[\frac{1}{q_{\TT_{1,k}+1}}-1\right]
\end{align*}
Given $\F_{\TT_{1,k}}$, let $\Theta_{k}$ denote a $\NN\left(\widehat{\mu}_{1,(\TT_{1,k})},\frac{1}{k}\right)$ variable. Let $G_k$ be the geometric random variable denoting the number of consecutive independent trials until and including the trial where a sample of $\Theta_{k}$ becomes greater $y_2$. Then observe that $q_{\TT_{1,k}+1}=\prob\left(\Theta_k>y_2|\F_{\TT_{1,k}}\right)$ and $\E\left[\frac{1}{q_{\TT_{1,k}+1}}\right]=\E\left[\E\left[G_k|\F_{\TT_{1,k}}\right]\right] =\E\left[G_k\right]$.\\
Let us abbreviate $\widehat{\mu}_{1,(\TT_{1,k})}$ as simply $\widehat{\mu}_1(k)$ ($\widehat{\mu}_1(k)$ is the mean of the past $k$ rewards corresponding to arm $1$). The following lemma provides an upper bound for $\E[G_k]$.
\begin{lemma}
    \label{lem:regret-geom-bdd}
    For $L(\horizon)=\frac{288\log\left(\horizon\Delta_{2,T}^2+300\right)}{\Delta_{2,T}^2}$, it holds that
    \begin{align*}
        \E[G_k]\le\begin{cases}
            7\times10^4+2&\forevery\; k,\\
            1+\frac{5}{\horizon\Delta_{2,T}^2}&k>L(\horizon)
        \end{cases}
    \end{align*}
\end{lemma}
\begin{proof}
    For any $r\ge1$, let $\text{MAX}_r$ denote the the maximum of $r$ independent samples of $\Theta_k$.
    \begin{align*}
        \prob(G_k\le r)&\ge \prob(\text{MAX}_r> y_2)\\
        &\ge \prob\left(\text{MAX}_r>\widehat{\mu}_1(k)+\sqrt{\frac{\log(r)}{k}} \ge y_2\right)\\
        &=\E\left[\E\left[\I\left\{\text{MAX}_r>\widehat{\mu}_1(k)+\sqrt{\frac{\log(r)}{k}} \ge y_2\right\}\Big|\;\F_{\TT_{1,k}}\right]\right]\\
        &=\E\left[\I\left\{\widehat{\mu}_1(k)+\sqrt{\frac{\log(r)}{k}} \ge y_2\right\}\prob\left(\text{MAX}_r>\widehat{\mu}_1(k)+\sqrt{\frac{\log(r)}{k}}\; \Big|\;\F_{\TT_{1,k}}\right)\right]
    \end{align*}
    Given $\F_{\TT_{1,k}}$, $\text{MAX}_r$ is the maximum of $r$ many independent Gaussian with mean $\widehat{\mu}_1(k)$ and variance $1/k$. Therefore,
    \begin{align*}
        \prob \left(\text{MAX}_r > \widehat{\mu}_1(k)+ \sqrt{\frac{\log(r)}{k}}\; \Big|\; \F_{\TT_{1,k}}\right)&\ge1-\left(\prob\left(\Theta_k\le \widehat{\mu}_1(k)+\sqrt{\frac{\log(r)}{k}}\; \Big|\;\F_{\TT_{1,k}}\right)\right)^r\\
        &=1-\left(\prob\left(\NN(0,1)\le\sqrt{\log(r)}\right)\right)^r
    \end{align*}
    Using the lower bound in \eqref{lem:Phi-&-phi-bdd}, we have
    \begin{align*}
        1-\left(\prob\left(\NN(0,1)\le\sqrt{\log(r)}\right)\right)^r&\ge1-\left(1-\frac{1}{\sqrt{2\pi}}\cdot\frac{\sqrt{\log(r)}}{\log(r)+1}\cdot\exp\left\{-\frac{\left(\sqrt{\log(r)}\right)^2}{2}\right\}\right)^r\\
        &=1-\left(1-\frac{1}{\sqrt{2\pi}}\cdot\frac{\sqrt{\log(r)}}{\log(r)+1}\cdot\frac{1}{\sqrt{r}}\right)^r
    \end{align*}For \textit{large} $r$, in particular, for $r\ge 7\times10^4$,
    \begin{align*}
        1-\left(1-\frac{1}{\sqrt{2\pi}}\cdot\frac{\sqrt{\log(r)}}{\log(r)+1}\cdot\frac{1}{\sqrt{r}}\right)^r\ge1-\exp\left\{-\frac{r}{\sqrt{4\pi r\log(r)}}\right\}\ge 1-\frac{1}{r^2}
    \end{align*}
    This implies for $r\ge 7\times10^4$
    \begin{align*}
        \prob(G_k\le r)&\ge\E\left[\I\left\{\widehat{\mu}_1(k)+\sqrt{\frac{\log(r)}{k}} \ge y_2\right\}\left(1-\frac{1}{r^2}\right)\right]\\
        &\ge\left(1-\frac{1}{r^2}\right)\prob\left(\widehat{\mu}_1(k)+\sqrt{\frac{\log(r)}{k}} \ge y_2\right)\\
        &=\left(1-\frac{1}{r^2}\right)\prob\left(\widehat{\mu}_1(k)-\mu_1 \ge y_2-\mu_1-\sqrt{\frac{\log(r)}{k}}\right)\\
        &=\left(1-\frac{1}{r^2}\right)\prob\left(\widehat{\mu}_1(k)-\mu_1 \ge -\left(\frac{\Delta_{2,T}}{3}+\sqrt{\frac{\log(r)}{k}}\right)\right)\\
        &\ge\left(1-\frac{1}{r^2}\right)\prob\left(\widehat{\mu}_1(k)-\mu_1 \ge -\sqrt{\frac{\log(r)}{k}}\right)\\
        &=\left(1-\frac{1}{r^2}\right)\prob\left(\NN(0,\sigma_T^2) \ge -\sqrt{\log(r)}\right)\\
        &=\left(1-\frac{1}{r^2}\right)\prob\left(\NN(0,1) \ge -\frac{\sqrt{\log(r)}}{\sigma_T}\right)
    \end{align*}
    Since $\sigma_T\le\frac{1}{2}$, using the upper bound in \eqref{lem:Phi-&-phi-bdd}, we have
    \begin{align*}
        \left(1-\frac{1}{r^2}\right)\prob\left(\NN(0,1) \ge -\frac{\sqrt{\log(r)}}{\sigma_T}\right)&\ge\left(1-\frac{1}{r^2}\right)\prob\left(\NN(0,1) \ge -\sqrt{4\log(r)}\right)\\
        &\ge \left(1-\frac{1}{r^2}\right)\left(1-\frac{\phi(\sqrt{4\log(r)}}{\sqrt{4\log(r)}}\right)\\
        &\ge\left(1-\frac{1}{r^2}\right)\left(1-\frac{r^{-2}}{\sqrt{8\pi\log(r)}}\right)\\
        &\ge 1-\frac{2}{r^2}
    \end{align*}
    Therefore, for all $r\ge 7\times10^4$,
    \begin{align}
        \label{eq:geom-quad-decay}
        \prob(G_k\le r)\ge1-\frac{2}{r^2}
    \end{align}
    Hence, for all $k$
    \begin{align}
        \label{eq:geom-const-bdd}
        \E[G_k]=\sum_{r=0}^\infty \prob(G_k\ge r)=\sum_{r<7\times10^4}\frac{2}{r^2}+\sum_{r\ge 7\times10^4}\frac{2}{r^2}<7\times10^4+2
    \end{align}
    
     \noindent Next, we prove a \textit{tighter} bound when $k$ is large, specifically, when $k>L(\horizon)$. We proceed similarly.
    \begin{align*}
        \prob(G_k\le r)&\ge \prob(\text{MAX}_r> y_2)\\
        &\ge\prob\left(\text{MAX}_r>\widehat{\mu}_1(k)+\sqrt{\frac{\log(r)}{k}} -\frac{\Delta_{2,T}}{6}\ge y_2\right)\\
        &=\E\left[\E\left[\I\left\{\text{MAX}_r>\widehat{\mu}_1(k)+\sqrt{\frac{\log(r)}{k}}-\frac{\Delta_{2,T}}{6} \ge y_2\right\}\Big|\;\F_{\TT_{1,k}}\right]\right]\\
        &=\E\left[\I\left\{\widehat{\mu}_1(k)+\sqrt{\frac{\log(r)}{k}}+\frac{\Delta_{2,T}}{6} \ge \mu_1\right\}\prob\left(\text{MAX}_r>\widehat{\mu}_1(k)+\sqrt{\frac{\log(r)}{k}}-\frac{\Delta_{2,T}}{6}\; \Big|\;\F_{\TT_{1,k}}\right)\right]
    \end{align*}
    Note that $k>L(\horizon)$ implies $\sqrt{\frac{2\log\left(\horizon\Delta_{2,T}^2+300\right)}{k}}\le\frac{\Delta_{2,T}}{12}$ and hence, for $r\le \left(\horizon\Delta_{2,T}^2+300\right)^2$, we have
    \begin{align*}
        \sqrt{\frac{\log(r)}{k}}-\frac{\Delta_{2,T}}{6}\le-\frac{\Delta_{2,T}}{12}
    \end{align*}
    So,
    \begin{align*}
        \prob\left(\text{MAX}_r>\widehat{\mu}_1(k)+\sqrt{\frac{\log(r)}{k}}-\frac{\Delta_{2,T}}{6}\; \Big|\;\F_{\TT_{1,k}}\right)&\ge\prob\left(\text{MAX}_r>\widehat{\mu}_1(k)-\frac{\Delta_{2,T}}{12}\; \Big|\;\F_{\TT_{1,k}}\right)\\
        &=1-\left[1-\prob\left(\Theta_k>\widehat{\mu}_1(k)-\frac{\Delta_{2,T}}{12}\; \Big|\;\F_{\TT_{1,k}}\right)\right]^r
    \end{align*}
    Now, observe that
    \begin{align*}
        \prob\left(\Theta_k>\widehat{\mu}_1(k)-\frac{\Delta_{2,T}}{12}\; \Big|\;\F_{\TT_{1,k}}\right)&=\prob\left(\NN\left(0,\frac{\sigma_T^2}{k}\right)\ge-\frac{\Delta_{2,T}}{12}\right)\\
        &=1-\prob\left(\NN\left(0,\frac{\sigma_T^2}{k}\right)\ge\frac{\Delta_{2,T}}{12}\right)\\
        &\stackrel{(i)}{\ge} 1-\prob\left(\NN\left(0,\frac{1}{k}\right)\ge\frac{\Delta_{2,T}}{12}\right)\\
        &\stackrel{(ii)}{\ge} 1-\exp\left\{-\frac{k\Delta_{2,T}^2}{288}\right\}\\
        &\stackrel{(iii)}{\ge} 1-\exp\left\{-\frac{L(\horizon)\Delta_{2,T}^2}{288}\right\}\\
        &=1-\frac{1}{\left(\horizon\Delta_{2,T}^2+300\right)}
    \end{align*}
    The inequality $(i)$ holds since $\sigma_T<1$, the inequality $(ii)$ follows from concentration bound for Gaussian random variables~\cite[\textsc{Chapter 2}]{wainwright2019high}, and inequality $(iii)$ utilizes $k>L(\horizon)$. Therefore,
    \begin{align*}
        \prob\left(\text{MAX}_r>\widehat{\mu}_1(k)+ \sqrt{\frac{\log(r)}{k}}-\frac{\Delta_{2,T}}{6}\; \Big|\;\F_{\TT_{1,k}}\right)\ge1-\frac{1}{(\horizon\Delta_{2,T}^2+300)^r}
    \end{align*}
    Define $\horizon^\star:=(\horizon\Delta_{2,T}^2+300)^2$. For $r\le \horizon^\star$ and $k>L(\horizon)$, we have
    \begin{align*}
        \prob(G_k\le r)&\ge\E\left[\I\left\{\widehat{\mu}_1(k)+\sqrt{\frac{\log(r)}{k}}+\frac{\Delta_{2,T}}{6} \ge \mu_1\right\}\prob\left(\text{MAX}_r>\widehat{\mu}_1(k)+\sqrt{\frac{\log(r)}{k}}-\frac{\Delta_{2,T}}{6}\; \Big|\;\F_{\TT_{1,k}}\right)\right]\\
        &\ge \E\left[\I\left\{\widehat{\mu}_1(k)+\sqrt{\frac{\log(r)}{k}}+\frac{\Delta_{2,T}}{6} \ge \mu_1\right\}\left(1-\frac{1}{\left(\sqrt{\horizon^\star}\right)^r}\right)\right]\\
        &=\left(1-\frac{1}{\left(\sqrt{\horizon^\star}\right)^r}\right)\prob\left(\widehat{\mu}_1(k)+\sqrt{\frac{\log(r)}{k}}+\frac{\Delta_{2,T}}{6} \ge \mu_1\right)\\
        &\ge\left(1-\frac{1}{\left(\sqrt{\horizon^\star}\right)^r}\right)\prob\left(\widehat{\mu}_1(k)-\mu_1\ge-\frac{\Delta_{2,T}}{6}\right)\\
        &=\left(1-\frac{1}{\left(\sqrt{\horizon^\star}\right)^r}\right)\prob\left(\NN\left(0,\frac{\sigma_T^2}{k}\right)\ge-\frac{\Delta_{2,T}}{6}\right)\\
        &\ge\left(1-\frac{1}{\left(\sqrt{\horizon^\star}\right)^r}\right)\prob\left(\NN\left(0,\frac{1}{k}\right)\ge-\frac{\Delta_{2,T}}{6}\right)\\
        &\ge \left(1-\frac{1}{\left(\sqrt{\horizon^\star}\right)^r}\right)\left[1-\exp\left\{-\frac{k\Delta_{2,T}^2}{72}\right\}\right]\\
        &\ge \left(1-\frac{1}{\left(\sqrt{\horizon^\star}\right)^r}\right)\left[1-\exp\left\{-\frac{L(\horizon)\Delta_{2,T}^2}{72}\right\}\right]\\
        &=\left(1-\frac{1}{\left(\sqrt{\horizon^\star}\right)^r}\right)\left[1-\exp\left\{-4\log\left(\sqrt{\horizon^\star}\right)\right\}\right]\\
        &=\left(1-\frac{1}{\left(\sqrt{\horizon^\star}\right)^r}\right)\left(1-\frac{1}{\left(\horizon^\star\right)^2}\right)\\
        &\ge1-\frac{1}{\left(\sqrt{\horizon^\star}\right)^r}-\frac{1}{\left(\horizon^\star\right)^2}
    \end{align*}
    For $r\ge \horizon^\star> 7\times10^4$, from \eqref{eq:geom-quad-decay}, we have $\prob(G_k\le r)\ge 1-\frac{2}{r^2}$. So, for $k>L(\horizon)$
    \begin{align*}
        \E[G_k]&=\sum_{r=0}^\infty\prob(G_k\ge r)\\
        &\le 1+\sum_{r=1}^{\horizon^\star}\prob(G_k\ge r)+\sum_{r\ge \horizon^\star}\prob(G_k\ge r)\\
        &\le 1+\sum_{r=1}^{\horizon^\star}\left[\frac{1}{\left(\sqrt{\horizon^\star}\right)^r}+\frac{1}{(\horizon^\star)^2}\right]+\sum_{r> \horizon^\star}\frac{2}{r^2}\\
        &=1+\sum_{r=1}^{\horizon^\star}\left[\frac{1}{\left(\sqrt{\horizon^\star}\right)^r}+\frac{1}{(\horizon^\star)^2}\right]+2\sum_{r> \horizon^\star}\left[\frac{1}{r-1}-\frac{1}{r}\right]\\
        &\le1+\frac{2}{\sqrt{\horizon^\star}}+\frac{1}{\horizon^\star}+\frac{2}{\horizon^\star}\\
        &\le1+\frac{5}{\sqrt{\horizon^\star}}\\
        &\le1+\frac{5}{\horizon\Delta_{2,T}^2}
    \end{align*}
    This completes the proof of Lemma \ref{lem:regret-geom-bdd}.
\end{proof}
\noindent Therefore,
    \begin{align}
        \sum_{t=1}^\horizon \prob\left(A_t=2,E_\mu^{(t)},E_\theta^{(t)}\right)&\le\sum_{k=0}^{\horizon-1}\E\left[\frac{1}{q_{\TT_{1,k}+1}}-1\right]\nonumber\\
        &\le \left(7\times10^4+2\right)L(\horizon)+\frac{5}{\Delta_{2,T}^2}\nonumber\\
        &=\left(7\times 10^4+2\right)\cdot\frac{\log(\horizon\Delta_{2,T}^2+300)}{\Delta_{2,T}^2}+\frac{5}{\Delta_{2,T}^2}\label{eq:bdd-term-a}
    \end{align}
\subsubsection{Bound for \ref{termB}}
Let $l(\horizon):=\frac{18\log\left(\horizon\Delta_{2,T}^2\right)}{\Delta_{2,T}^2}$. Rewrite \ref{termB} as
\begin{align*}
    \sum_{t=1}^\horizon \prob\left(A_t=2,\; E_\mu^{(t)},\; \overline{E_\theta^{(t)}}\right)&=\sum_{t=1}^\horizon \prob\left(A_t=2,\; n_{2,t}\le l(\horizon),\; E_\mu^{(t)},\; \overline{E_\theta^{(t)}}\right)\\
    &\hspace{1cm}+\sum_{t=1}^\horizon \prob\left(A_t=2,\; n_{2,t}> l(\horizon),\; E_\mu^{(t)},\; \overline{E_\theta^{(t)}}\right)
\end{align*}
Clearly, the first term on the right hand side is bounded by $l(\horizon)$ itself. For the second term
\begin{align*}
    \sum_{t=1}^\horizon \prob\left(A_t=2,\; n_{2,t}> l(\horizon),\; E_\mu^{(t)},\; \overline{E_\theta^{(t)}}\right)&=\E\left[\sum_{t=1}^\horizon\prob\left(A_t=2,\; \overline{E_\theta^{(t)}}\; \Big|\: n_{2,t}> l(\horizon),\; E_\mu^{(t)},\; \F_{t-1}\right)\right]\\
    &\le\E\left[\sum_{t=1}^\horizon\prob\left(\theta_2^{(t)}>y\; \Big|\: n_{2,t}> l(\horizon),\; \widehat{\mu}_{2,t}\le x_2,\; \F_{t-1}\right)\right]
\end{align*}
Given $\F_{t-1}$, $\theta_2^{(t)}$ follows Gaussian with mean $\widehat{\mu}_{2,t}$. Given $\widehat{\mu}_{2,t}\le x_2$, the distribution of $\theta_2^{(t)}$ is \textit{stochastically dominated}~\footnote{We say $W_2$ is stochastically dominated by $W_1$ if for all $x\in\R$, $\prob(W_1>x)\ge\prob(W_2>x)$, for random variables $W_1$ and $W_2$.} by $\NN\left(x_2,\frac{1}{n_{2,t}}\right)$. So,
\begin{align*}
    \prob\left(\theta_2^{(t)}>y_2\; \Big|\: n_{2,t}> l(T),\; \widehat{\mu}_{2,t}\le x_2,\; \F_{t-1}\right)\le\prob\left(\NN\left(x_2,\frac{1}{n_{2,t}}\right)>y_2\; \Big|\: n_{2,t}> l(T),\; \widehat{\mu}_{2,t}\le x_2,\; \F_{t-1}\right)
\end{align*}
Note that
\begin{align*}
    \prob\left(\NN\left(x_2,\frac{1}{n_{2,t}}\right)>y_2\right)&=\prob\left(\NN\left(0,\frac{1}{n_{2,t}}\right)>\frac{\Delta_{2,T}}{3}\right)\\
    &\le\exp\left\{-\frac{n_{2,t}\Delta_{2,T}^2}{18}\right\}
\end{align*}
Since, $n_{2,t}>l(\horizon)=\frac{18\log\left(\horizon\Delta_{2,T}^2\right)}{\Delta_{2,T}^2}$, the above is bounded by $\frac{1}{\horizon\Delta_{2,T}^2}$. Hence,
\begin{align*}
    \E\left[\sum_{t=1}^\horizon\prob\left(\theta_2^{(t)}>y_2\; \Big|\: n_{2,t}> l(\horizon),\; \widehat{\mu}_{2,t}\le x_2,\; \F_{t-1}\right)\right]\le\frac{1}{\Delta_{2,T}^2}
\end{align*}
Therefore,
\begin{align}
    \sum_{t=1}^\horizon \prob\left(A_t=2,\; E_\mu^{(t)},\; \overline{E_\theta^{(t)}}\right)\le\frac{18\log(\horizon\Delta_{2,T}^2)}{\Delta_{2,T}^2}+\frac{1}{\Delta_{2,T}^2}\label{eq:bdd-term-b}
\end{align}
\subsubsection{Bound for \ref{termC}}
Recall that $\TT_{2,k}$ denote the time of playing arm $2$ for the $k^{\text{th}}$ time. Since $\TT_{2,\horizon}\ge \horizon$, we have
\begin{align}
    \sum_{t=1}^\horizon \prob\left(A_t=2,\overline{E_\mu^{(t)}}\right)&\le\sum_{k=1}^\horizon\prob\left(\overline{E_\mu^{(\TT_{2,k})}}\right)\nonumber\\
    &=\sum_{k=1}^\horizon \prob\left(\widehat{\mu}_{2,\TT_{2,k}}>x_2\right)\nonumber\\
    &\le\sum_{k=1}^\horizon\prob\left(\NN\left(\mu_2,\frac{\sigma_T^2}{k}\right)>x_2\right)\nonumber\\
    &=\sum_{k=1}^\horizon\prob\left(\NN\left(0,\frac{\sigma_T^2}{k}\right)>\frac{\Delta_{2,T}}{3}\right)\nonumber\\
    &\le\sum_{k=1}^\horizon\prob\left(\NN\left(0,\frac{1}{k}\right)>\frac{\Delta_{2,T}}{3}\right)\nonumber\\
    &\le \sum_{k=1}^\horizon \exp\left\{-\frac{k\Delta_{2,T}^2}{18}\right\}\nonumber\\
    &\le\frac{18}{\Delta_{2,T}^2}\label{eq:bdd-term-c}
\end{align}
Since arm $2$ was selected arbitrarily among the sub-optimal arms, therefore, the bounds in \eqref{eq:bdd-term-a}, \eqref{eq:bdd-term-b} and \eqref{eq:bdd-term-c} hold for any arm $j$ with $\Delta_{j,T}>0$. Summing up, we get the desired result. This completes the proof of Lemma~\ref{lem:regret-bd-shipra-goyal}.

\section{Auxiliary Lemmas}
\label{sec:aux-lemmas}
In this section, we prove the state and prove auxiliary lemmas and corollaries used in the main part of the paper and proofs in Appendices. 

\begin{lemma}
    \label{lem:geom-conv-0-a} For all $n\in\N$, let $\left\{p_j^{(n)}\right\}_{j=1}^m$ be in $(0,1)$, where $m\le n^2$ and $p_1^{(n)}\ge c$ for some $c>0$. Let $G_j^{(n)}$ be independent geometric random variables with success probability $p_j^{(n)}$. Then, for any $a>0$
    \begin{align*}
        \frac{\max_{1\le j\le m} \log\left(G_j^{(n)}p_j^{(n)}\right)}{n^a}\inprob0
    \end{align*}
\end{lemma}
\begin{proof}
    First observe that for all $n\in\N$,
    \begin{align}
        \label{geom-lemma-lb}
        \max_{1\le j\le m}\log\left(G_j^{(n)}p_j^{(n)}\right) \ge \log\left(G_1^{(n)}p_1^{(n)}\right)\ge \log(c)
    \end{align}
    Secondly, for $G\sim\text{Geometric}(p)$ where $p\in(0,1)$, $\prob(G\le x)=1-(1-p)^{\left\lfloor x\right\rfloor}$ for $x>0$. Using this result, we have
    \begin{align*}
        \prob\left(\max_{1\le j\le m}\log\left(G_j^{(n)}p_j^{(n)}\right)> n^{a/2}\right)&=1-\prob\left(\max_{1\le j\le m}\log\left(G_j^{(n)}p_j^{(n)}\right)\le n^{a/2}\right)\\
        &=1-\prob\left(\log\left(G_j^{(n)}p_j^{(n)}\right)\le n^{a/2}\; \forevery\; j\in[m]\right)\\
        &=1-\prod_{j\in[m]} \prob\left(\log\left(G_j^{(n)}p_j^{(n)}\right)\le n^{a/2}\right)\\
        &=1-\prod_{j\in[m]} \prob\left(G_j^{(n)}p_j^{(n)}\le\exp\left\{ n^{a/2}\right\}\right)\\
        &=1-\prod_{j\in[m]}\left[1-\left(1-p_j^{(n)}\right)^{\left\lfloor\frac{\exp\left\{ n^{a/2}\right\}}{p_j^{(n)}}\right\rfloor}\right]
    \end{align*}
    Since $\lfloor x\rfloor>x-1>\frac{x}{2}$ for all $x>2$, it follows
    \begin{align}
        \label{eq:lem-a-geom-prob-ub}
        \left(1-p_j^{(n)}\right)^{\left\lfloor\frac{\exp\left\{n^{a/2}\right\}}{p_j^{(n)}}\right\rfloor}\le \left(1-p_j^{(n)}\right)^{\left(\frac{\exp\left\{n^{a/2}\right\}}{2p_j^{(n)}}\right)}\le\exp\left\{-\frac{\exp\left\{ n^{a/2}\right\}}{2}\right\}
    \end{align}
    The last inequality holds because $(1-x)^{1/x}\le 1/e$ for all $x\in(0,1)$. Furthermore, for all large $n$, $\exp\left\{n^{a/2}\right\}>6\log(n)$ and hence,
    \begin{align*}
        \exp\left\{-\frac{\exp\left\{n^{a/2}\right\}}{2}\right\}\le\frac{1}{n^3}
    \end{align*}
    Thus, for all large $n$
    \begin{align*}
        1-\prod_{j\in[m]}\left[1-\left(1-p_j^{(n)}\right)^{\left\lfloor\frac{\exp\left\{n^{a/2}\right\}}{p_j^{(n)}}\right\rfloor}\right]&\le1-\left[1-\frac{1}{n^3}\right]^m\\
        &\le1-\left[1-\frac{1}{n^3}\right]^{n^2}\\
        &\le 1-\left(\frac{1}{e}\right)^\frac{1}{n}
    \end{align*}
    Therefore, for all large $n$
    \begin{align}
        \label{geom-lemma-ub}
        \prob\left(\max_{1\le j\le m}\log\left(G_j^{(n)}p_j^{(n)}\right)>n^{a/2}\right)&\le1-\left(\frac{1}{e}\right)^\frac{1}{n}
    \end{align}
    From \eqref{geom-lemma-lb} and \eqref{geom-lemma-ub}, it follows that for all $\varepsilon>0$ and for all large $n$
    \begin{align*}
        &\prob\left(\left|\frac{\max_{1\le j\le m}\log\left(G_j^{(n)}p_j^{(n)}\right)}{n^a}\right|> \varepsilon\right)\\
        =\;& \prob\left(\max_{1\le j\le m}\log\left(G_j^{(n)}p_j^{(n)}\right)> \varepsilon n^a\right)+\prob\left(\max_{1\le j\le m}\log\left(G_j^{(n)}p_j^{(n)}\right)<-\varepsilon n^a\right)\\
        \le\;& \prob\left(\max_{1\le j\le m}\log\left(G_j^{(n)}p_j^{(n)}\right)>  n^{a/2}\right)+\prob\left(\max_{1\le j\le m}\log\left(G_j^{(n)}p_j^{(n)}\right)<\log(c)\right)\\
        \le\;& 1-\left(\frac{1}{e}\right)^\frac{1}{n}
    \end{align*}
    which converges to $0$ as $n\to\infty$. Hence, the lemma follows.
\end{proof}
\begin{lemma}
    \label{lem:geom-conv-0-b} For all $n\in\N$, let $G_n$ be geometric with success probability $p_n\in(0,1)$. Then, for any $a>0$
    \begin{align*}
        \frac{\log\left(G_n p_n\right)}{n^a}\inprob0
    \end{align*}
\end{lemma}
\begin{proof}
    Let us consider $\prob\left(G_n p_n>\frac{1}{n}\right)=\left(1-p_n\right)^{\left\lfloor\frac{1}{n p_n}\right\rfloor}$. If $p_n\in[1/2,1)$, then for $n>2$, $\frac{1}{n p_n}<1$ and hence this probability equals $1$. For $p_n<1/2$,
    \begin{align*}
        \left(1-p_n\right)^{\left\lfloor\frac{1}{n p_n}\right\rfloor}\ge\left(1-p_n\right)^{\frac{1}{n p_n}}=\left[\left(1-p_n\right)^{\frac{1}{p_n}}\right]^{1/n}\ge\frac{1}{4^{1/n}}
    \end{align*}
    where the last inequality holds because $(1-x)^{1/x}\ge1/4$ for $x\in(0,1/2)$. So, for all $n>2$, we conclude
    \begin{align}
        \label{eq:lem-b-geom-prob-lb}
        \prob\left(G_n p_n>\frac{1}{n}\right)\ge\frac{1}{4^{1/n}}
    \end{align}
    Now, for all $\varepsilon>0$, for all $a>0$,
    \begin{align*}
        \prob\left(\left|\frac{\log\left(G_n p_n\right)}{n^a}\right|>\varepsilon\right)&=\prob\left(\frac{\log\left(G_n p_n\right)}{n^a}>\varepsilon\right)+\prob\left(\frac{\log\left(G_n p_n\right)}{n^a}<-\varepsilon\right)\\
        &=\prob\left(G_n p_n>\exp\left\{\varepsilon n^a\right\}\right)+\prob\left(G_n p_n<\exp\left\{-\varepsilon n^a\right\}\right)
    \end{align*}
    For all large $n$, using \eqref{eq:lem-a-geom-prob-ub}, it follows
    \begin{align}
        \prob\left(G_n p_n>\exp\left\{\varepsilon n^a\right\}\right)&\le\prob\left(G_n p_n>\exp\left\{ n^{a/2}\right\}\right)\nonumber\\
        &\le \exp\left\{-\frac{\exp\left\{ n^{a/2}\right\}}{2}\right\}\label{lem-b-term-1-ub}
    \end{align}
    For all large $n$, using \eqref{eq:lem-b-geom-prob-lb}, it follows
    \begin{align}
        \prob\left(G_n p_n<\exp\left\{-\varepsilon n^a\right\}\right)&\le\prob\left(G_n p_n\le\frac{1}{n}\right)\le1-\frac{1}{4^{1/n}}
        \label{lem-b-term-2-ub}
    \end{align}
    Clearly the upper bounds in \eqref{lem-b-term-1-ub} and \eqref{lem-b-term-2-ub} converges to $0$ as $n\to\infty$. Therefore,
    \begin{align*}
        \limsup_{n\to\infty}\prob\left(\left|\frac{\log\left(G_n p_n\right)}{n^a}\right|>\varepsilon\right)\le0
    \end{align*}
    and hence the result holds.
\end{proof}
\begin{lemma}
    \label{lem:bd-implies-conv}
    Let $\{X_n\}_{n\ge1},\{Y_n\}_{n\ge1},\{Z_n\}_{n\ge1}$ be positive random variables such that they sum up to 1 and $\E[Z_n]\rightarrow 0$ as $n \rightarrow \infty$. Suppose on some event $\{A_n\}_{n\ge1}$ with $\prob(A_n)\to1$ as $n\to\infty$, we have  
    \begin{align*}
        X_n,Y_n\le \frac{1}{2}+a_n
    \end{align*}
    where $a_n$ is a non-negative sequence such that $a_n\downarrow0$ as $n \rightarrow \infty$. Then, 
    \begin{align*}
        X_n\inprob\frac{1}{2}\qquad Y_n\inprob\frac{1}{2}\qquad Z_n\inprob0
    \end{align*}
\end{lemma}
\begin{proof}
Since $Z_n\ge 0$, then on the event $A_n$, it holds that $X_n-\frac{1}{2}\le a_n \le a_n+Z_n$. Now, for the lower bound, observe that $X_n-\frac{1}{2}=\frac{1}{2}-Y_n-Z_n\ge-a_n-Z_n$. Thus, on the event $A_n$,
\begin{align*}
    \left|X_n-\frac{1}{2}\right|\le a_n+Z_n
\end{align*}
Therefore, for any $\varepsilon>0$, 
\begin{align*}
    \prob\left(\left|X_n-\frac{1}{2}\right|\ge\varepsilon\right)&=\prob\left(\left|X_n-\frac{1}{2}\right|\ge\varepsilon;\; A_n\right)+\prob\left(\left|X_n-\frac{1}{2}\right|\ge\varepsilon;\; A_n^{\text{c}}\right)\\
    &\le\prob\left(a_n+Z_n\ge\varepsilon;\; A_n\right)+\prob\left(A_n^{\text{c}}\right)\\
    &\le\prob\left(a_n+Z_n\ge\varepsilon\right) +\prob\left(A_n^{\text{c}}\right)
\end{align*}
Next, $\E[Z_n]\rightarrow 0$ implies $Z_n\inprob0$ and hence, by Slutsky's Theorem, we conclude $a_n+Z_n\inprob0$. Thus, $X_n\inprob\frac{1}{2}$ and by similar arguments $Y_n\inprob\frac{1}{2}$.
\end{proof}
\begin{lemma}
    \label{additional-lemmas}
    We state and prove some algebraic results used in the proof of Theorem~\ref{main-thm}.\\
    \begin{subequations}
    Let $a_1,a_2,\ldots$ be a real-valued sequence. Then, for any $n$,
        \begin{align}
            \label{lem:log-exp-bdd}
            \max_{i\in[n]}a_i\le \log\left(\sum_{i\in[n]}e^{a_i}\right)\le \max_{i\in[n]}a_i+\log(n)
        \end{align}
    For any $z>0$
        \begin{align}
            \label{lem:Phi-&-phi-bdd}
            \frac{z\phi(z)}{1+z^2}\le1-\Phi(z)\le\frac{\phi(z)}{z}
        \end{align}
    \end{subequations}
\end{lemma}
\noindent\textbf{Proof of}~\eqref{lem:log-exp-bdd}: This follows directly from the observation that
\begin{align*}
    \exp\left\{\max_{i\in[n]}a_i\right\}\le \sum_{i\in[n]}e^{a_i}\le n\exp\left\{\max_{i\in[n]}a_i\right\}
\end{align*}

\noindent\textbf{Proof of}~\eqref{lem:Phi-&-phi-bdd}: For the upper bound, note that $1-\Phi(z)=\int_z^\infty \phi(u)\text{d}u\le\int_z^\infty \frac{u}{z}\cdot\phi(u)\text{d}u=\frac{\phi(z)}{z}$. For the other side, $1-\Phi(z)=\int_z^\infty\phi(u)\text{d}u\ge\int_z^\infty \frac{z^2}{1+z^2}\cdot\frac{u^2+1}{u^2} \cdot\phi(u)\text{d}u=\frac{z\phi(z)}{z^2+1}$.
\end{document}